\documentclass[twoside]{article}

\usepackage[accepted]{aistats2025}

\usepackage{amsthm,amssymb}
\usepackage{algorithm}
\usepackage{algorithmic}
\usepackage{graphicx}
\usepackage{subcaption}
\usepackage{adjustbox}
\newtheorem*{remark}{Note}

\usepackage{xcolor}
\usepackage{tcolorbox}
\usepackage{marginnote}

\newtheoremstyle{thmstyle}
  {0pt}   
  {+0pt}   
  {\itshape}  
  {}       
  {\bfseries} 
  {.}      
  { }      
  {}       

\makeatletter
\renewenvironment{proof}[1][\proofname]{\par
  \pushQED{\qed}%
  \normalfont\topsep=0pt \partopsep=0pt 
  \trivlist
  \item[\hskip\labelsep
        \itshape
    #1\@addpunct{.}]\ignorespaces
}{%
  \popQED\endtrivlist\@endpefalse
}
\makeatother

\theoremstyle{thmstyle}
\newtheorem{theorem}{Theorem}[section]
\newtheorem{proposition}[theorem]{Proposition}
\newtheorem{lemma}[theorem]{Lemma}
\newtheorem{corollary}[theorem]{Corollary}
\newtheorem{definition}[theorem]{Definition}

\DeclareMathOperator*{\argmax}{arg\,max}
\DeclareMathOperator{\E}{\mathbb{E}}
\DeclareRobustCommand{\exp}[1]{\operatorname{exp}\left(#1\right)}
\newcommand{\adv}[2]{\textrm{adv}(#1,#2)}

\usepackage[round]{natbib}

\begin{document}

%

%
\runningauthor{ Arsenii Mustafin, Aleksei Pakharev, Alex Olshevsky, Ioannis Ch. Paschalidis}

\twocolumn[

\aistatstitle{MDP Geometry, Normalization and Reward Balancing Solvers}

\aistatsauthor{ Arsenii Mustafin$^1$ 
\And Aleksei Pakharev$^2$ 
\And  Alex Olshevsky$^1$ \And Ioannis Ch. Paschalidis$^1$}
\vspace{5pt}
\aistatsaddress{ $^1$ Boston University
  \And  
  $^2$ Memorial Sloan Kettering Cancer Center \\
 } ]

\begin{abstract}
We present a new geometric interpretation of Markov Decision Processes (MDPs) with a natural normalization procedure that allows us to adjust the value function at each state without altering the advantage of any action with respect to any policy. This advantage-preserving transformation of the MDP motivates a class of algorithms which we call {\em Reward Balancing}, which solve MDPs by iterating through these transformations, until an approximately optimal policy can be trivially found.  We provide a convergence analysis of several algorithms in this class, in particular showing that for MDPs for unknown transition probabilities we can improve upon state-of-the-art sample complexity results. 
\end{abstract}

\section{INTRODUCTION}

A {\em Markov Decision Process (MDP)} is a common mathematical model for sequential decision making. It was initially introduced in the late 1950s with main algorithms such as Value Iteration \citep{bellmandp} and Policy Iteration \citep{howard1960dynamic}. The topic has been developed over the course of the next decades, and the most important results were summarized in 1990 by \citet{puterman1990markov}. More recently, the rise of practical {\em Reinforcement Learning (RL)} algorithms has drawn new attention to MDP analyses, leading to several important results, especially in terms of finite-time convergence analysis. At the same time, we have not seen fundamentally new algorithms for solving MDPs and state value estimation remains a backbone of existing MDP algorithms.

In this paper, we introduce a novel geometric view of MDP actions. Inspired by this view we suggest a new class of algorithms which solve MDPs without assigning state values. We present a particular algorithm from this class, {\em Safe Reward Balancing (RB-S)}, that achieves state-of-the-art convergence results in several settings.

\subsection{Related work}
There is a vast and growing literature providing theoretical analyses of MDP algorithms. In this section we mention only those works, which are directly related to our paper.

We are aware of only one work which suggests an MDP analysis from a geometric perspective \citep{vfpolytope}, which was subsequently utilized by \cite{bellemare2019geometric,dabney2021value}. 

Value Iteration algorithm (VI) was initially proposed by \cite{bellmandp} as an algorithm to solve known MDPs, and its convergence was proven by \cite{howard1960dynamic} based on $\gamma$-contraction properties of the Bellman update (where $\gamma$ is the MDP discount factor). More recently, additional properties of VI convergence were explored in \cite{puterman2014markov, feinberg2014value}.

$Q$-learning might be seen as a version of VI applied to MDPs with unknown transition dynamics~\citep{watkins1989learning, watkins1992q}. Its asymptotic convergence was analyzed in \cite{jaakkola1993convergence, tsitsiklis1994asynchronous, szepesvari1997asymptotic, borkar2000ode}. Lately, a few finite-time convergence results have appeared in \citep{even2001learning, beck2012error, wainwright2019stochastic, chen2020finite}, including \citet{Q-learning_sota} showing sample complexity of $\tilde{\mathcal{O}}\left(\frac{1}{(1-\gamma)^4}\right)$\footnote{The notation $\tilde{O}(\cdot)$ subsumes constant and logarithmic factors.}, which is the current state-of-the-art result. Federated $Q$-learning was analysed in \cite{chen2022sample, shen2023towards, woo2023blessing}, with the latter work obtaining the current state-of-the-art sample complexity of $\tilde{\mathcal{O}}\left(\frac{1}{K(1-\gamma)^5}\right)$.

\subsection{Main Contributions} \label{ssec:contribution}

In this work we introduce a new geometric interpretation of an MDP and demonstrate how this perspective helps to analyze existing algorithms and develop a new algorithmic approach for solving MDPs. Our main contributions are as follows:
\vspace{-6pt}
\begin{itemize}
    \item We present a new geometric interpretation of the classical MDP framework and show that the MDP problems of policy evaluation and finding the optimal policy have geometric representations.   
    Building on this interpretation, we introduce a geometric transformation that modifies the MDP action rewards into so-called normal form while preserving the optimal policy. For an MDP in {\em normal form}, the problem of finding an optimal policy becomes trivial
    (Section~\ref{sec:geom_action_space}).
    \vspace{-4pt}
    \item Leveraging this intuition, we propose to solve MDPs by doing what we call {\em Reward Balancing}, which modifies an MDP to approach its normal form. We show that value iteration and several of its modifications can be viewed as doing this. We consider one algorithm in the Reward balancing class,  {\em Safe Reward Balancing (RB-S)}, and extend it to the unknown MDP, resulting in an an algorithm which is new to our knowledge. 
     The stochastic version of RB-S is learning-rate-free (i.e., no step-size needs to be set), maintains only a single vector at each iteration just like $Q$-learning, while achieving a state-of-the-art $O\left((1-\gamma)^{-4} \right)$ sample complexity for producing an $\epsilon$-optimal policy (and not just computing an approximation to the true $Q$-values). 
    Additionally, stochastic RB-S is parallelizable and one can speed up performance by a multiplicative factor of $K$ with $K$ independent workers sampling from the same MDP; this is not known to be the case for Q-learning (Section~\ref{sec:stoch_rb}).
\end{itemize}
\vspace{-10pt}

\section{Basic MDP setting}
We assume a standard infinite-horizon, discounted reward MDP setting \citep{sutton2018reinforcement, puterman2014markov}. An MDP is defined by the tuple $\mathcal{M} = \langle \mathcal{S}, \mathcal{A}, \mathrm{st}, \mathcal{P}, \mathcal{R}, \gamma \rangle$, where $\mathcal{S} = \{s_1,\ldots, s_n\}$ is a finite set of $n$ states; $\mathcal{A}$ is a finite set of size $m$ of all possible actions over all states combined; $\mathcal{P}$ is the assignment of a probability distribution $P(\,\cdot \,|\, a)$ over $\mathcal{S}$ to each action; $\mathcal{R}: \mathcal{A} \to \mathbb{R}$ specifies deterministic rewards given for actions in $\mathcal{A}$; and $\gamma \in (0,1)$ is a discount factor. The only non-standard attribute of $\mathcal{M}$ is $\mathrm{st}$, with which we denote the map from $\mathcal{A}$ to $\mathcal{S}$ that associates an action $a$ can be chosen at state $s$ to the state $s$, and we need it for the following reason.
\begin{remark}
We define $\mathcal{A}$ such that each action is attributed to a particular state via the map $\mathrm{st}$, \textit{i.e.,} if in two different states there is an option to choose two actions with the same rewards and the same transition probability distribution, we consider them as two different actions. In other words, we treat as separate actions what is called "state-action pairs" in earlier literature. Consequently, the overall number of actions $m = |\mathcal{A}|$ in our setting corresponds to $|\mathcal{S}|\cdot|\tilde{\mathcal{A}}|$ in the ``universal'' action setting, where $\tilde{\mathcal{A}}$ is the number of universal actions, i.e., actions available at every state.
\end{remark}

Having an MDP, an agent chooses a policy to interact with it, that is a map $\pi: \mathcal{S} \rightarrow \mathcal{A}$ such that $\mathrm{st}(\pi(s)) = s$ for any $s\in\mathcal{S}$. We consider only deterministic stationary policies, which are policies that choose a single action to perform in each state throughout the entire trajectory. If policy $\pi$ chooses action $a$ in a state $s$, we say that $a \in \pi$. Each policy can be defined by the actions it takes in all states, $\pi = \{a_1, \dots, a_n\}$, where $a_i$ denotes the action taken in state $i$.

Given a policy $\pi$, we can define its value over a state $s$ to be the expected infinite discounted reward over trajectories starting from the state $s$:
\begin{equation*}
V^\pi (s) = \E \left[ \sum_{t=1}^{\infty} \gamma^t r_t \right],    
\end{equation*} 
where $r_t$ is the reward received at time step $t$ starting with state $s$ and actions taken according to policy $\pi$. The vector $V^\pi$ will be the unique vector that satisfies the Bellman equation $T^\pi V^\pi = V^\pi$, where $T^\pi$ is the Bellman operator, defined as:
\[ (T^\pi V)(s) =R(\pi(s)) + \sum_{s'} P\left(s'\,|\,\pi(s)\right) \gamma V(s'). \]
Evaluating the policy values of a given policy $\pi$ is an important problem in MDPs called the Policy Evaluation problem. However, the main problem of interest given an MDP is to find the optimal policy $\pi^*$ such that:
\[ V^{\pi^*} (s) \ge V^\pi(s), \quad \forall \pi, s. \]
The exact solution to this problem can be found using the {\em Policy Iteration} algorithm, but each iteration requires inverting a square matrix of size $n \times n$, which is computationally expensive. Instead, we might be interested in a non-exact solution, or an $\epsilon$-optimal policy $\pi^\epsilon$~--- a policy that satisfies:
\[ V^{\pi^*} (s) - V^{\pi^\epsilon}(s) < \epsilon, \quad \forall s. \]
An $\epsilon$-optimal policy can be found using the VI algorithm in a number of iterations dependent on $\epsilon$, with each iteration having a computational complexity of $\mathcal{O}(n m)$.

\section{GEOMETRY OF ACTION SPACE} \label{sec:geom_action_space}

\begin{figure*}[!t] 
\begin{center}
\includegraphics[width=\textwidth]{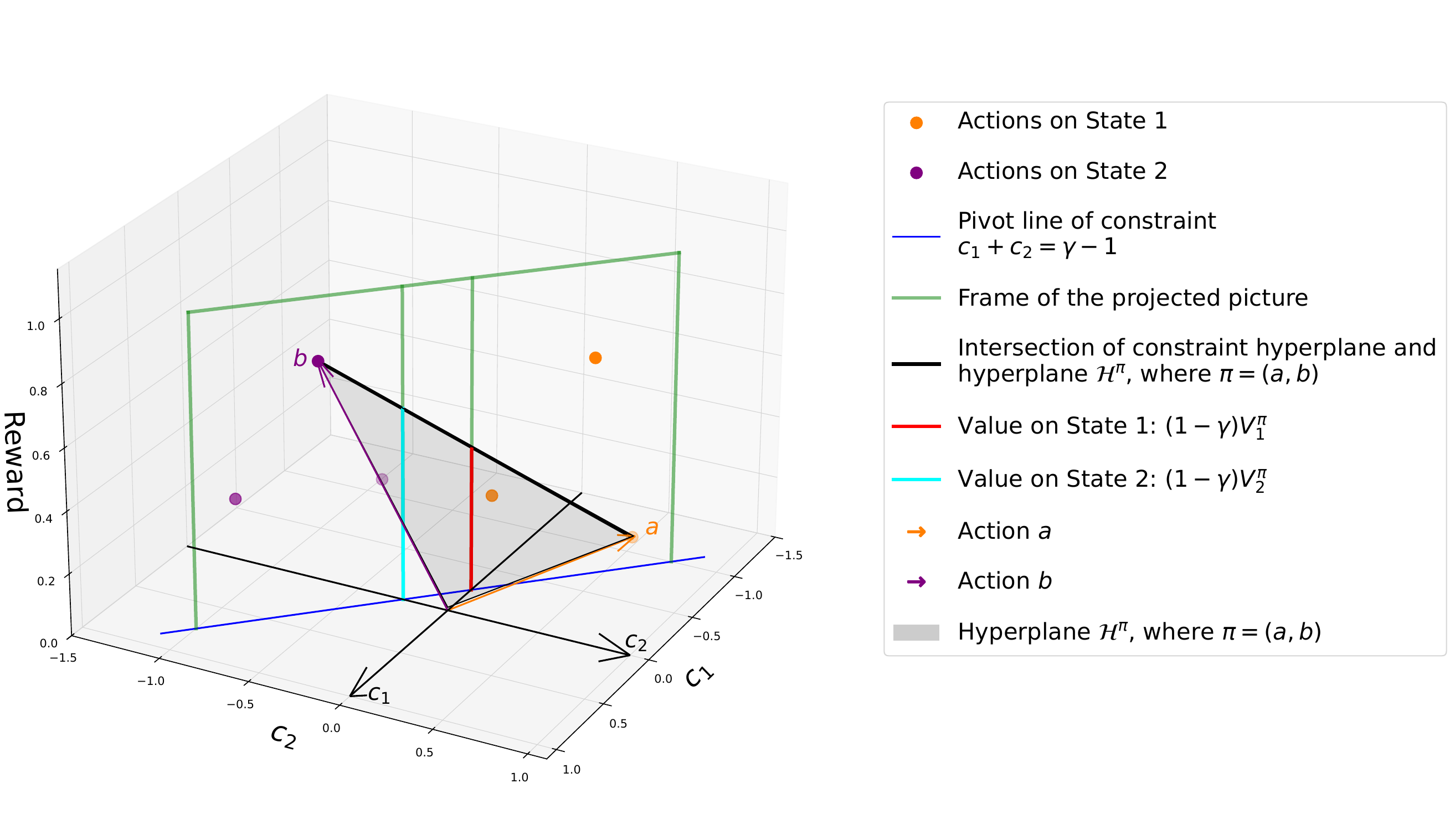}%
\end{center}
\caption{An example of the action space for a 2-state MDP with 3 actions in each state. The vertical axis is the action reward axis, while the two horizontal axes correspond to the first (axis $c_1$) and second (axis $c_2$) coefficients of the action vector (the same example with actions only can be found in the Appendix, Figure \ref{fig:action_space_example}). The figure also illustrates the application of Theorem \ref{thm:selfloop_values}.  The shaded area represents the policy hyperplane $\mathcal{H}^\pi$ of the policy $\pi = (a, b)$. The black line connecting actions $a$ and $b$ indicates the intersection of the policy hyperplane and the action constraint hyperplane. The red and cyan bar heights correspond to the values of the policy $\pi$ in States 1 and 2.}
\label{fig:action_space_with_policy}
\vspace{-10pt}
\end{figure*}
Take an action $a$ at state $1$ participating in a policy $\pi$ with values $V^\pi$. Denote the reward associated with this action as $r_a$, and the probabilities to move to state $s$ as $p^a_s$. We know that the policy values satisfy the Bellman equation, which can be rewritten as:
\begin{equation} \label{eq:action_equation}
r^a\cdot 1 + (\gamma p^a_1-1)V^\pi(1) + \gamma \sum_{i=2}^{n} p^a_i V^\pi(i) = 0.
\end{equation}
The left hand side can be written as the dot product of two vectors: $(r^a, \gamma p^a_1-1, \gamma p^a_2, \dots, \gamma p^a_n)$ and $(1, V^\pi(1), V^\pi(2), \dots, V^\pi(n))$. Note that the former vector contains all the information related to the action $a$~--- reward and transition probabilities~--- and can be constructed without defining a policy. The latter vector consists of policy values and can be constructed without knowing the actions participating in the policy. We denote the first vector as $a^+$ and call it the \textbf{action vector}, assuming it is a row vector, and the second vector as $V^\pi_+$ and call it the \textbf{policy vector}, assuming it is a column vector. With this notation we can establish 
the following proposition.

\begin{proposition}\label{pro:orthogonal_action_policy}
For a policy $\pi$ and an action $a \in \pi$, the dot product of the corresponding action vector and policy vector is $0$. In other words, the action vector $a^+$ is \textbf{orthogonal} to the policy vector $V^\pi_+$ if $a \in \pi$.
\end{proposition}

Consider the $(n+1)$-dimensional linear space where the action vectors lie, which we call the \textbf{action space}. Index the coordinates starting from $0$, so that an action vector $a^+$ would have its zeroth coordinate equal to the reward $r^a$. The possible locations of action vectors in this space are limited. First, the sum of all vector entries apart from the zeroth one is equal to $\gamma - 1$, which means that all of them lie on a fixed $n$-dimensional affine hyperplane. Additionally, exactly one of the last $n$ entries of an action vector $a^+$ must be negative, specifically the one corresponding to the state $\mathrm{st}(a)$, and all other entries are non-negative. Therefore, the actions of an MDP lie in the designated pairwise non-intersecting polytope zones in the action space, with each state having its own zone. For a state $s$, any corresponding action vector $a^+=(c(0), c(1), \ldots, c(n))$ must obey the following inequalities:
\begin{equation}\label{equ:designated_zone}\begin{cases}
\begin{aligned}
&c(1) + \cdots + c(n) = \gamma - 1, \\
&c(1), \ldots, c(s-1), c(s+1),\ldots, c(n) \geq 0,\\
&c(s) < 0.
\end{aligned}
\end{cases}\end{equation}
Figure \ref{fig:action_space_with_policy} shows an example of an action space for a 2-state MDP (and a more clear figure with actions only is given in Appendix \ref{sec:equ_and_dynamics}, Figure \ref{fig:action_space_example}). We suggest to visualize the reward, the zeroth coordinate in the action space, as the height function, as presented in the figure. We will refer to this choice of a height function later when we talk about horizontal and vertical directions or the relation of being higher or lower.

From the point of view of the action space, the policy vector $V^\pi_+$ defines the hyperplane of vectors orthogonal to $V^\pi_+$. Denote this hyperplane by $\mathcal{H}^\pi$. From Proposition~\ref{pro:orthogonal_action_policy}, it follows that for any action $a\in \pi$ in the policy $\pi$, the corresponding action vector $a^+$ lies in the hyperplane $\mathcal{H}^\pi$. Therefore, an alternative definition of $\mathcal{H}^\pi$ for a policy $\pi = \{a_1,\ldots, a_n\}$ is the linear hull of the action vectors $a_1^+, \ldots, a_n^+$. This allows us to restate the Bellman equation in more geometric terms.
\begin{proposition} \label{pro:policy_eval}
The problem of evaluating the values of a policy $\pi = \{a_1, \dots, a_n\}$ (Policy Evaluation Problem) is equivalent to the problem of calculating the coordinates of a vector normal to the $n$-dimensional linear span $\mathcal{H}^\pi = \langle a_1^+, \ldots, a_n^+ \rangle$ of the participating action vectors $a_1^+, \ldots, a_n^+$.
\end{proposition}
\begin{proof}
$\implies$: If the policy values are known, Proposition~\ref{pro:orthogonal_action_policy} implies that vector $(1, V^\pi(1), \dots, V^\pi(n) )$ is normal to the span $\mathcal{H}^\pi$.

$\impliedby$: Having the coefficients $(c(0),c(1),\ldots, c(n))$ of a vector normal to hyperplane $\mathcal{H}^\pi$, the values of the policy $\pi$ can be calculated as $(c(1), c(2), \ldots, c(n))/c(0)$. Indeed, the coefficients of this hyperplane are defined uniquely up to a scaling, and are equal to the values if the zeroth coordinate is equal to $1$.
\end{proof}
The important property of this interpretation is that the values of the policy $\pi$ defining the hyperplane $\mathcal{H}^\pi$ can be recovered in an intuitive geometric way, presented in Figure \ref{fig:action_space_with_policy}. Recall that for a state $s$ we have a designated zone for action vectors defined by ~\ref{equ:designated_zone}. Consider the following one-dimensional line in the boundary of the zone: $(r,0,\ldots,0,\gamma-1,0,\ldots,0)$, where the zeroth coordinate $r$ is equal to a free parameter, and the $s$-th coordinate is equal to $\gamma-1$. Denote this line by $L_s$. In fact, this line represents all vectors of the self-loop actions in state $s$.

\begin{theorem}\label{thm:selfloop_values}
The values of a policy $\pi$ can be recovered as the rewards of the intersection points between $\mathcal{H}^\pi$ and the self-loop lines $L_s$, $s\in\mathcal{S}$, divided by $1-\gamma$.
\end{theorem}
\begin{proof}
For each state $s$, take the point $l_s$ of the intersection between $L_s$ and $\mathcal{H}^\pi$. We know that the vector $l_s$ is orthogonal to the vector $V^\pi_+$:
\begin{gather*}
0 = l_s V^\pi_+ = l_s(0) + \sum_{s'=1}^n l_s(s')V^\pi(s') = \\
= l_s(0) + (\gamma-1) V^\pi(s) \implies
 V^\pi(s) = l_s(0)/(1-\gamma).
\end{gather*}
Thus, the zeroth coordinate of $l_s$ corresponds to the policy value in state $s$.\end{proof}

Theorem \ref{thm:selfloop_values} simplifies the analysis by allowing us to focus solely on the constrained space $\sum_{i=1}^n c(i) = \gamma - 1$  as all meaningful information, including both actions and policy values, is contained within it. Consequently, Figure \ref{fig:action_space_with_policy} can be reduced to Figure \ref{fig:policy_2d_projection}. The theorem helps to view policies in geometric terms in the action space: every policy can be thought of as a hyperplane in this space, and any non-vertical hyperplane (\textit{i.e.}, non-parallel to the reward axis) can be thought of as a policy substitute, but not necessarily produced by a set of available actions. Such hyperplanes, or sets of values which are not produced by the actions in the MDP, often appear in practice, for example, during the run of the Value Iteration algorithm. 
To underline their similarity to regular policies, we call them pseudo-policies.

\begin{figure*}[t!] 
\begin{center}
\includegraphics[width=\textwidth]{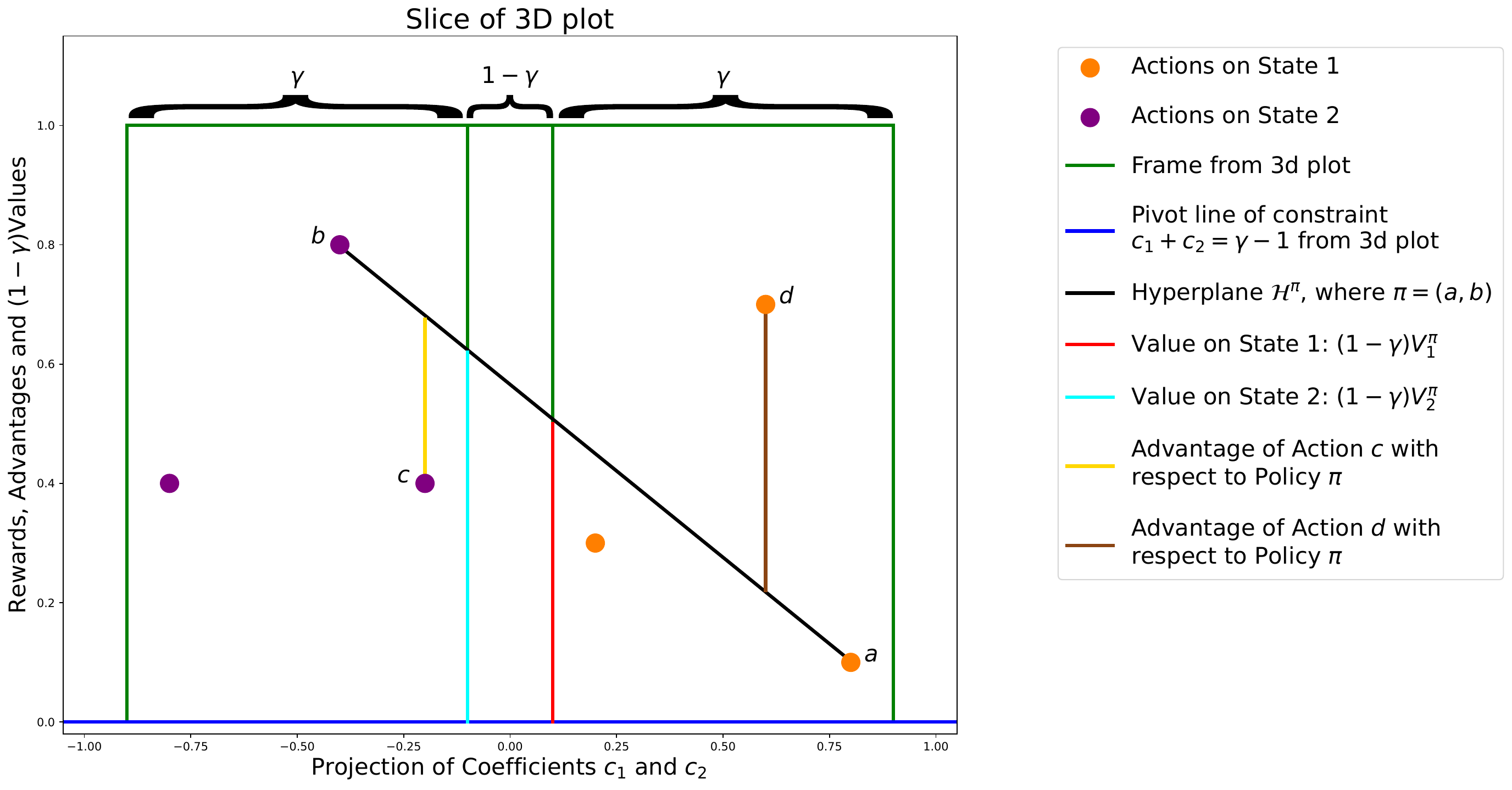}%
\end{center}
\caption{Two-dimensional plot corresponding to the constrained action space from Figure \ref{fig:action_space_with_policy}.  All the information required to analyze the MDP~--- action coefficients and policy values~--- is presented in the plot. Additionally, the plot includes the advantages of two actions, $c$ and $d$, that do not participate in the policy $\pi$.}
\label{fig:policy_2d_projection}
\vspace{-10pt}
\end{figure*}

\subsection{Optimal policy in action space} \label{sec:opt_policy}
Consider an action $b$ and a policy $\pi$ such that $b \notin \pi$, assuming without loss of generality that $b$ is an action on the state 1. What is the dot product of the action vector $b^+$ and the policy vector $V^\pi_+$? Mathematically, it is:
\begin{equation*}
b^+V^\pi_+ = r^b + \gamma \sum_{i=1}^{n} p^b_i V^\pi(i) - V^\pi(1).
\end{equation*}
This quantity is often used in the reinforcement learning literature and is known as an \textbf{advantage} or an \textbf{appeal} of action $b$ with respect to policy $\pi$, denoted as $\adv{b}{\pi}$ (in this paper we stick to the term advantage). It has an intuitive meaning: the quantity shows the potential gain (or loss) in the infinite sum of discounted rewards if, instead of following the policy $\pi$ from State 1, the agent takes action $b$ and then follows $\pi$.

Advantages are key quantities in MDPs since they define the dynamics of the main MDP algorithms:
\vspace{-4pt}
\begin{itemize}
\item Given a policy $\pi_t$, the \textbf{Policy Iteration} algorithm chooses the action $a_s$ in a state $s$ with the highest advantage among other actions in $s$ with respect to $\pi_t$.
\vspace{-3pt}\item Similarly, while choosing actions to update values, the \textbf{Value Iteration} algorithm chooses actions with the highest advantage with respect to the pseudo-policy implied by current values.
\end{itemize}
\vspace{-4pt}
Now let's establish the relation between an action advantage and the policy hyperplane in the action space. Since the advantage of an action $b$ with respect to a policy $\pi$ is the inner product of the vector $b^+$ with the vector $V^\pi_+$, which is orthogonal to the hyperplane $\mathcal{H}^\pi$, a positive advantage means that the vector $b^+$ lies above $\mathcal{H}^\pi$, and a negative advantage means that $b^+$ lies below $\mathcal{H}^\pi$. With this observation, we are ready to present the second important result of this paper.

\begin{theorem} \label{thm:opt_policy_hplane}
The problem of \textbf{solving an MDP}, or finding a stationary optimal policy $\pi^*$ for an MDP, is equivalent to the problem of constructing a hyperplane $\mathcal{H}$ such that at least one action from every state is incidental to it and no action lies above it.
\end{theorem}
\begin{proof}
$\implies$: Suppose we know the optimal policy $\pi^*$. Every action has $0$ or negative advantage with respect to it; therefore, none of the action vectors lie above the hyperplane $\mathcal{H}^{\pi^*}$. On the other hand, each state has the optimal action, and the corresponding action vector is incidental to $\mathcal{H}^{\pi^*}$.

$\impliedby$: Let's choose a set of actions $(a_1, \ldots, a_n)$ such that their action vectors are incidental with $\mathcal{H}$ and denote the policy they form as $\pi$. Clearly $\mathcal{H} = \mathcal{H}^\pi$. We know that $a_i^+ V^\pi_+ = 0$ and $b^+ V^\pi_+ \le 0, \forall \mathrm{st}(b) = s $. Then $a_i = \argmax_{b} \adv{b}{\pi}$, which implies that policy $\pi$ is stationary with respect to a greedy upgrade and, consequently, optimal. \end{proof}

We can view Theorem \ref{thm:opt_policy_hplane}
as a rephrasing of the linear programming formulation of finding an optimal policy in an MDP. However, this more geometric rephrasing,  together with Proposition \ref{pro:policy_eval} and Theorem \ref{thm:selfloop_values}, provides a foundation for the approach to MDP solving that we develop in this paper.

\subsection{MDP transformation} \label{sec:mdp_transform}

If we try to analyze the dynamics of the policy iteration algorithm in a 2-state MDP (which we do in Appendix \ref{ssec:policy_iteration} and Figure \ref{fig:policy_iteration_dynamics}), we see that it is not very convenient to do the analysis when the rewards for two states are not comparable (e.g., large rewards in one state and small rewards in another state), making all policy hyperplanes steep. Can we modify the MDP to make the corresponding hyperplanes flatter? It turns out we can! In this section, we describe the transformation $\mathcal{L}$ of an MDP which changes the values of every policy in one particular state, while keeping all advantages unchanged.

Choose a state $s$ and a real number $\delta \in \mathbb{R}$~--- the desired increment of the state value $V^\pi (s)$ of every policy $\pi$. We denote this transformation as $\mathcal{L}_s^\delta$. To perform it, we need to modify \textbf{every reward in the MDP} using the following rules:
\vspace{-6pt}
\begin{itemize}
\item For every action $a, \mathrm{st}(a) = s$ the reward $r^a$ should be modified as:
\begin{equation*}
   r^a := r^a - \delta(\gamma p^a_s - 1). 
\end{equation*} 
\vspace{-3pt}\item For every action $b, \mathrm{st}(b) \ne s$ the reward $r^b$ should be modified as:
\begin{equation*} r^b := r^b - \delta \gamma p^b_s, \end{equation*} 
\end{itemize}
\vspace{-4pt}
where $p^a_s$ and $p^b_s$ are the probabilities of getting to state $s$ from the state of the action.

The geometric interpretation of $\mathcal{L}_s^\delta$ is as follows. We transform the whole action space with the linear transformation
$$
(x(0), x(1), \ldots, x(n)) \mapsto (x(0) - \delta x(s), x(1), \ldots, x(n)).
$$
This transformation maps the designated action zones for each state to itself. Any point on a line $L_{s'}$, $s'\neq s$, is fixed by this transformation, and any point on the line $L_{s}$ is lifted by $\delta(1-\gamma)$. Coupled with Theorem~\ref{thm:selfloop_values}, we see that the values of any hyperplane $\mathcal{H}$ stay the same for $s'\neq s$, and increases by $\delta$ on the state $s$. The following theorem shows the most important property of $\mathcal{L}$.

\begin{theorem} \label{thm:adv_preservation}
Transformation $\mathcal{L}^{\delta}_s$ preserve the advantage of any action with respect to any policy.    
\end{theorem}

\begin{proof}
A simple proof is given in Appendix \ref{ssec:adv_preservation_proof}.    
\end{proof}

For MDPs $\mathcal{M}$ and $\mathcal{M}'$, we write $\mathcal{M}' = \mathcal{L}_s^\delta \mathcal{M}$ if $\mathcal{L}_s^\delta$ transforms $\mathcal{M}$ to $\mathcal{M}'$. Since these transformations are mutually commutative, it is useful to introduce the following notation:
\begin{equation*} \mathcal{M}' = \mathcal{L}^\Delta \mathcal{M} = \mathcal{L}_{s_1}^{\delta_1}(\mathcal{L}_{s_2}^{\delta_2}( \dots( \mathcal{L}_{s_k}^{\delta_k}\mathcal{M})\dots )), \end{equation*}
where $\Delta = (\delta_1, \delta_2, \dots, \delta_n)^T$ is a column vector, and the $i$th entry of it, $\delta_i$, is the quantity added to every value in state $s_i$.

Then, choose some MDP $\mathcal{M}$ and compute its optimal values $(V^*(1), \dots, V^*(n))$. Apply the following sequence of transformations to $\mathcal{M}$:
\begin{equation} \label{eq:mdp_normalization}
\mathcal{M}^* = \mathcal{L}^{\Delta^*} \mathcal{M},
\end{equation}
where $\Delta^* = (-V^*(1), \dots, -V^*(n))^T$. The values of the optimal policy $\pi^*$ in $\mathcal{M}^*$ are 0 in all states. At the same time, since the transformations $\mathcal{L}$ preserve advantages, the advantages of the modified actions $a' \in \pi^*$ remain $0$, which implies that
\begin{equation*}
\adv{a'}{\pi^*}= 
 r^{a'} + \sum_{i=1}^{n} c^a_i V^{\pi'} (i) = r^{a ' }= 0,
\end{equation*}
meaning that every action participating in the optimal policy in $\mathcal{M}^*$ has reward $0$, while every action not participating in the optimal policy has a reward equal to its advantage with respect to the optimal policy, which is always negative. Geometrically, the optimal policy hyperplane $\mathcal{H}^{\pi^*}$ is given by the equation $c_0 = 0$, all the action vectors from the optimal policy lie on the hyperplane, while every action not participating in the optimal policy lie below it.

This picture is convenient to analyze, so we give this special kind of MDPs a definition:
\begin{definition}
We call an MDP $\mathcal{M}^*$ \textbf{normal} if all values of the optimal policy of $\mathcal{M}^*$ are equal to zero.

The \textbf{normalization} $\mathcal{M}^*$ of an MDP $\mathcal{M}$ is the result of applying the transformation $L^{\Delta^*}$ to the MDP $\mathcal{M}$, where $\Delta^*=-V^*$, negation to optimal policy values of MDP $\mathcal{M}$. The normalization of an MDP is always a normal MDP.
\end{definition}
 
In the next section, we use normal MDPs to analyze the dynamics of the main RL algorithms.

\section{REWARD BALANCING ALGORITHMS} \label{sec:rb_solvers}

There is one more nice property of a normalization of an MDP that we did not mention in Section \ref{sec:mdp_transform}: it is trivially solvable! To solve it, one just needs to pick the action with the reward of 0 at every state, and the resulting policy is optimal. This policy has a value of 0 at every state, while any other action has a negative reward, thus their inclusion to a policy will result in negative values.

This is an intriguing property of an MDP. Imagine one needs to solve an MDP $\mathcal{M}$. Suppose that $\mathcal{M}$ can be solved by Policy Iteration starting from a policy $\pi_0$ in $T$ steps. If the same algorithm is run on the normalization $\mathcal{M}^*$, it will take exactly the same number of iterations $T$ to solve it as it is shown in Section \ref{ssec:policy_iteration}. However, there is no need to run the algorithm on the normalization: we can just pick actions with rewards $0$ to form the optimal policy. This property of the normalization of an MDP inspires a new class of solution algorithms that do not compute state values but instead solve the MDP by manipulating its rewards. Due to this characteristic, we refer to these value-free algorithms as \textbf{reward balancing} solvers.

Of course, given a random MDP $\mathcal{M}$, it is not possible to normalize it without solving it, because the optimal values are not known. However, we still have the transformations $\mathcal{L}_s^\delta$ available to us, and we have a goal: if at some point we achieve an MDP in which the maximum action reward in every state is 0, we know that the corresponding actions form the optimal policy. Furthermore, the following lemma shows that if the MDP is ``close'' to normal, a policy consisting of actions with the maximum reward at each state produces a good approximate solution.

\begin{lemma} \label{lem:approxite_solution}
Suppose that all rewards of an MDP $\mathcal{M}$ are non-positive and larger than $r_{min}$.
If a policy $\pi$ chooses the actions with the maximum reward in each state, then $\pi$ is $\epsilon$-optimal with $\epsilon = -r_{min}/(1-\gamma)$.
\end{lemma}

\begin{proof}
    We know that the values of any policy in this MDP are upper bounded by 0. At the same time, minimum possible values of the policy satisfy $||V^* - V^\pi||_\infty < (0-r_{\min})/(1-\gamma)$.
\end{proof}

Therefore, the quantity $r_{min}$ can be used to establish a stopping criterion. Now we are ready to provide an intuition for a solution algorithm.

\begin{remark}
   We start every algorithm by decreasing the reward of every action in the MDP by the overall maximum reward so that the maximum action reward becomes $0$ and every action has a non-positive reward. 
\end{remark}
Our goal is then to carefully "flatten" action rewards by applying transformations $\mathcal{L}_{s}^{\delta}$ to the states where the highest rewards are below 0, so that the maximum reward in every state becomes closer to $0$. The issue is that while increasing the reward of an action $a$ in a state $s$, the transformation $\mathcal{L}_{s}^{\delta}$ also has a negative side effect: it decreases the rewards of all actions in every other state that have a non-zero transition probability to $s$. A straightforward approach would be to account for these side effects, but this algorithm is equivalent to the Policy Iteration algorithm, as it is discussed in Appendix \ref{ssec:exact_rb_solver}.

However, if we do not aim to produce an exact solution, we can develop a novel algorithm, which is not similar to any Value-based approximate algorithms. The easiest way to simplify the algorithm is to disregard the aforementioned side effects and make ``safe" updates, which leads to Algorithm \ref{alg:simple_vf_alg} --- Safe Reward Balancing or RB-S for short \footnote{The authors thank Julien Hendricks for suggesting this version of the algorithm.}. More information about the intuition behind the design of the algorithm is given in Appendix \ref{ssec:rb_intuition}.

\setlength{\textfloatsep}{10pt}
\begin{algorithm}[t]
   \caption{Safe Reward Balancing (RB-S)}
   \label{alg:simple_vf_alg}
\begin{algorithmic}
   \STATE {\bfseries Initialize} $\mathcal{M}_0 := \mathcal{M}$ and $t=1$.
   \STATE \textbf{Update:} \begin{itemize}
       \item $\forall s$ compute $\delta_s = - \max_a r^a / (1 - \gamma p^a_s)$.
       \item Update the MDP $\mathcal{M}_t = \mathcal{L}^{\Delta_t} \mathcal{M}_{t-1}$, where $\Delta_t = (\delta_1, \dots, \delta_n)^T$.
       \item Store the maximum action reward in each state while performing the update. Set $R^m_t$ to the minimum of these maxima. 
   \end{itemize} 
      
   \IF{ $|R^m_t| /(1-\gamma) \ge \epsilon$}
    \STATE Increment $t$ by 1 and return to the update step.
   \ELSE
    \STATE Output policy $\pi$, such that $\pi(s) = \argmax_{a}  r^a$.
    \ENDIF
\end{algorithmic}
\end{algorithm}

Algorithm \ref{alg:simple_vf_alg}, which in the known MDP setting might be seen as the Value Iteration on a modified MDP, has the same iteration complexity as VI, but has different convergence properties. In particular, VI does not perform well on tree-structured MDPs, potentially achieving the worst possible convergence rate of $\gamma$. In contrast, tree-structured MDPs are the best case for the RB-S algorithm. We can show this for a broader case of hierarchical MDPs.

\begin{definition}
MDP $\mathcal{M}$ is \textbf{hierarchical} if every state can be attributed to a class $c \in \mathbb{N}$ such that every action leads either back to the state $s$ or to states that belong to lower classes.
\end{definition}

\begin{theorem} \label{thm:hier_conv}
Applied to a hierarchical MDP $\mathcal{M}$, the RB-S algorithm is guaranteed to converge to the exact solution in at most $C$ iterations where $C$ is the number of hierarchy classes.
\end{theorem}
\begin{proof}
An induction-based proof of the theorem is given in Appendix \ref{ssec:hier_conv_proof}.
\end{proof}
The RB-S algorithm has interesting convergence properties not only in hierarchical MDPs, but on general MDPs too, which is shown in the following theorem.
\begin{theorem} \label{thm:rb_conv}
The policy $\pi_t$ is $\epsilon_t$-optimal where $$\epsilon_t = \alpha^t \frac{l}{(1-\beta)(1-\gamma)} r_{\max}.$$ Here
\begin{gather*} \alpha = \max_{s, \pi} \frac{\gamma - \gamma [P_{\pi}]_{ss}}{1 - \gamma [P_{\pi}]_{ss}} \leq \gamma, \quad \beta = \max_{s,\pi} \gamma [P_{\pi}]_{ss} \leq \gamma, \\
 l = 1 - \gamma \min_{s,\pi} [P_{\pi}]_{ss} \leq 1,
\end{gather*}
\end{theorem}
and 
$$
r_{\max} = -\min_s \max_{a, \mathrm{st}(a) = s} r^a_0
$$ 
is the negation of the minimum among maximum state rewards.
\begin{proof}
The proof is given in Appendix \ref{ssec:rb_conv_proof}
\end{proof}
The theorem implies that convergence rate $\alpha$ of RB-S algorithm is always smaller or equal than $\gamma$. Notably, RB-S algorithm enjoys faster convergence rate in the case of non-trivial self-loop action probabilities in the MDP. Also, the theorem has a following interesting corollary:
\begin{corollary} \label{cor:adv_rew_conv}
For all actions $a$,
\vspace*{-1em}
\[ \left| r_t^a - \adv{a}{\pi^*} \right|  \leq 2 r_{\max} \frac{\gamma^t}{1-\gamma}.   \]
\end{corollary} 
\begin{proof}
The Corollary proof follows from the proof of Theorem \ref{thm:stoch_conv} given in Appendix \ref{ssec:stoch_conv_proof}.
\end{proof}
The significance of the corollary is that the upper bound is observable during the run of the algorithm, which allows for the RB-S algorithm to output an exact solution by applying action filtering, see Appendix \ref{ssec:rb_w_af} for details.

The main difference between RB-S and VI in the known MDP case is how they treat self-loop actions. In the absence of the self-loops two algorithms are equivalent as it is shown in Appendix \ref{ssec:diag_free_sim}.
Appendix \ref{ssec:experiments_known} demonstrates the experimental comparison of the two algorithms.

\subsection{MDPs with unknown dynamics} \label{sec:stoch_rb}

In Section \ref{sec:rb_solvers} we considered the case of fully known MDPs, where all transition probabilities $\mathcal{P}$ are known and available to the solver. In practice, the assumption of a fully known MDP is usually impractical, and the solver only has access to transition data in the form $(s, s', a, r^a)$, where $a$ is an action chosen in state $s$, $r^a$ is the received reward, and $s'$ is the state where the agent appears after a transition sampled from $P(s'\,|\,a)$. 

In this section we switch from state-based to action-based notation. From now on, the reward vector $r_t$ is of length equal to the number of actions $m$ and consists of all action rewards from every state. Consequently, the transition matrix $P^\pi$ has a size of $m \times m$, and its element $P[i,j]$ is the probability that after choosing action $i$, the next action chosen is action $j$. This value is non-zero only if the agent has a non-zero probability of appearing in state $s' = \mathrm{st}(j)$, and policy $\pi$ chooses action $j$ in $s'$. For our analysis, we need an additional vector $R_t$, a vector of length $m$, where for every action $a$, it stores the maximum reward available in the state of action $a$ at time $t$, $R_t(a) = \max_b r^b:\mathrm{st}(b) = \mathrm{st}(a)$. For convenience we abuse notation and define $R_t(s) = \max_b r^b :\mathrm{st}(b) = s$.

The unknown MDP setting also requires a few adjustments to the algorithm:
the first adjustment we need to make is to give up on self-loop coefficients, since self-loop probabilities are no longer known. Therefore, without the self-loop coefficients, but with known transition probabilities, the reward balancing update becomes
\begin{equation}\label{eq:rew_balancing_actions}
r_{t+1} = r_t - R_t + \gamma P R_t.
\end{equation}
However, we no longer have access to the transition probabilities $P^{\pi_t}$, but only to samples from it. We assume an access to generative model, which is common in the literature analyzing the convergence of the synchronous $Q$-learning algorithm. It means that the solver can freely sample any number of transitions for any action at any time step. For the stability of the algorithm, at step  $t$ it samples not one, but $k$ transitions matrices $P^i_t$ from $P$.

Then the algorithm, which we call {\em Stochastic Reward Balancing}, has an update:
\begin{equation}\label{eq:rb_sampling}
r_{t+1} = r_t - R_t + \gamma \left( \frac{P^1_t + \dots + P^k_t}{k} \right) R_t.
\end{equation}
\begin{remark}
Just like Q-learning, this algorithm requires to store only a single vector of dimension $m$.
\end{remark} 
We denote the estimation matrix above as $P_t$:
\begin{equation*}
    P_t = \frac{P^1_t + \dots + P^k_t}{k}, \quad \E [P_t] = P.
\end{equation*}
A formal definition of the algorithm is given in Appendix \ref{ssec:stoch_conv_proof}. The following theorem analyzes its convergence.

\begin{theorem} \label{thm:stoch_conv} Fix $\tau \in [0,1]$ and choose 
\[ k \ge \frac{4r_{\max}^2 \log \frac{2m}{1-\tau}}{\epsilon^2(1-\gamma)^3 (1+\gamma)} , \]
and run stochastic reward balancing for 
\[ t \ge \frac{1}{1-\gamma} \log \left( \frac{1}{\epsilon (1-\gamma)} \right), \] steps. 
Let $\pi_t$ be the policy where every node takes the action with the maximal reward at time $t$. Then $\pi_t$ is $\epsilon$-optimal with probability $1-\tau$. 
\end{theorem}

\begin{proof}
Proofs of a few propositions, lemmas and the theorem itself are given in Appendix \ref{ssec:stoch_conv_proof}.
\end{proof}

\begin{corollary}
Overall the number of samples $M$ required to achieve an $\epsilon$-optimal policy is:
\begin{equation*}
M \approx m \frac{2r_{\max}^2}{\epsilon^2 (1-\gamma)^4} \log(m) \log \left( \frac{1}{\epsilon (1-\gamma)} \right),    
\end{equation*}
which is a scale of $\tilde{\mathcal{O}} ( \frac{m}{(1-\gamma)^4\epsilon^2})$.
\end{corollary}
We now can compare this complexity with previous works. Note, that RB-S actually produces better results since it outputs $\epsilon$-optimal policy, while known Q-learning results only guarantee the convergence to a vector of Q values such that $||Q - Q^*||_{\infty} <\epsilon$. Obtaining $\epsilon$-optimal policy requires increasing the complexity by factor $(1-\gamma)^{-2}$. More comments on the difference between two algorithms are given in Appendix \ref{ssec:rbs_vi_stoch_diff}.

With the notation used in this paper, the current state-of-the-art result for $Q$-learning from \citet{Q-learning_sota} is 
\begin{equation*}
M \gtrapprox m \frac{c_3 (\log^4 M)(\log \frac{mM}{1-\tau}) }{\epsilon^2(1-\gamma)^4 }.     
\end{equation*}
Therefore, the result we provide for RB-S has the same scale as Q-learning, but improves upon it in terms of logarithmic and constant factors, while also eliminating the need for a learning rate parameter and any associated tuning.

\begin{corollary}
The sampling of transitions required to compute $P_t$ can be done independently which makes the algorithm parallelizable. Therefore, in a {\em federated learning} setting, splitting the task over $K$ agents provides a linear speed up, and the overall number of samples assigned to each agent is a scale of
\begin{equation*}
\tilde{\mathcal{O}} \left( \frac{m}{K(1-\gamma)^4\epsilon^2} \right).
\end{equation*}
Additionally, the algorithm requires $\tilde{\mathcal{O}}(1/(1-\gamma))$ communication rounds, which makes it efficient in terms of amount of communication needed.
\end{corollary}

This result improves upon the current state-of-the-art result on federated $Q$-learning from \citet{Parallel_q_learning_sota} by the factor of $(1-\gamma)$.

Overall, the Safe Reward Balancing algorithm is not only conceptually novel, but also exhibits favorable convergence properties, with sample efficiency that surpasses current results from $Q$-learning, while requiring exactly the same amount of space.

\section{CONCLUSION}
In this paper, we have presented a new geometric interpretation of Markov Decision Processes, establishing the equivalence between MDP problems and specific geometric problems. Based on this interpretation, we developed new algorithms for both known and unknown MDP cases, showing state-of-the-art convergence properties in some cases and improving on the state-of-the-art in others. The authors conjecture that these are only the first results provided by this new interpretation. The authors hope that this new perspective would be useful to the researchers in the field whether in the analysis of existing algorithms or the development of new ones.

\bibliographystyle{plainnat}
\bibliography{main}
\newpage

\section*{Checklist}

 \begin{enumerate}

 \item For all models and algorithms presented, check if you include:
 \begin{enumerate}
   \item A clear description of the mathematical setting, assumptions, algorithm, and/or model. \textbf{Yes}
   \item An analysis of the properties and complexity (time, space, sample size) of any algorithm. \textbf{Yes}
   \item (Optional) Anonymized source code, with specification of all dependencies, including external libraries. \textbf{Yes}
 \end{enumerate}

 \item For any theoretical claim, check if you include:
 \begin{enumerate}
   \item Statements of the full set of assumptions of all theoretical results. \textbf{Yes}
   \item Complete proofs of all theoretical results. \textbf{Yes}
   \item Clear explanations of any assumptions. \textbf{Yes}     
 \end{enumerate}

 \item For all figures and tables that present empirical results, check if you include:
 \begin{enumerate}
   \item The code, data, and instructions needed to reproduce the main experimental results (either in the supplemental material or as a URL). \textbf{Yes}
   \item All the training details (e.g., data splits, hyperparameters, how they were chosen). \textbf{Not Applicable}
         \item A clear definition of the specific measure or statistics and error bars (e.g., with respect to the random seed after running experiments multiple times). \textbf{Yes}
         \item A description of the computing infrastructure used. (e.g., type of GPUs, internal cluster, or cloud provider). \textbf{Not Applicable}
 \end{enumerate}

 \item If you are using existing assets (e.g., code, data, models) or curating/releasing new assets, check if you include:
 \begin{enumerate}
   \item Citations of the creator If your work uses existing assets. \textbf{Not Applicable}
   \item The license information of the assets, if applicable. \textbf{Not Applicable}
   \item New assets either in the supplemental material or as a URL, if applicable. \textbf{Not Applicable}
   \item Information about consent from data providers/curators. \textbf{Not Applicable}
   \item Discussion of sensible content if applicable, e.g., personally identifiable information or offensive content. \textbf{Not Applicable}
 \end{enumerate}

 \item If you used crowdsourcing or conducted research with human subjects, check if you include:
 \begin{enumerate}
   \item The full text of instructions given to participants and screenshots. \textbf{Not Applicable}
   \item Descriptions of potential participant risks, with links to Institutional Review Board (IRB) approvals if applicable. \textbf{Not Applicable}
   \item The estimated hourly wage paid to participants and the total amount spent on participant compensation. \textbf{Not Applicable}
 \end{enumerate}

 \end{enumerate}
\onecolumn
\appendix

\newpage

\section{ADDITIONAL REMARKS ON MDP GEOMETRY} \label{sec:equ_and_dynamics}

\subsection{Example of the Geometric Interpretation of a Regular MDP} \label{ssec:geom_example}

\begin{figure*}[t] 
\begin{center}
\includegraphics[width=\textwidth]{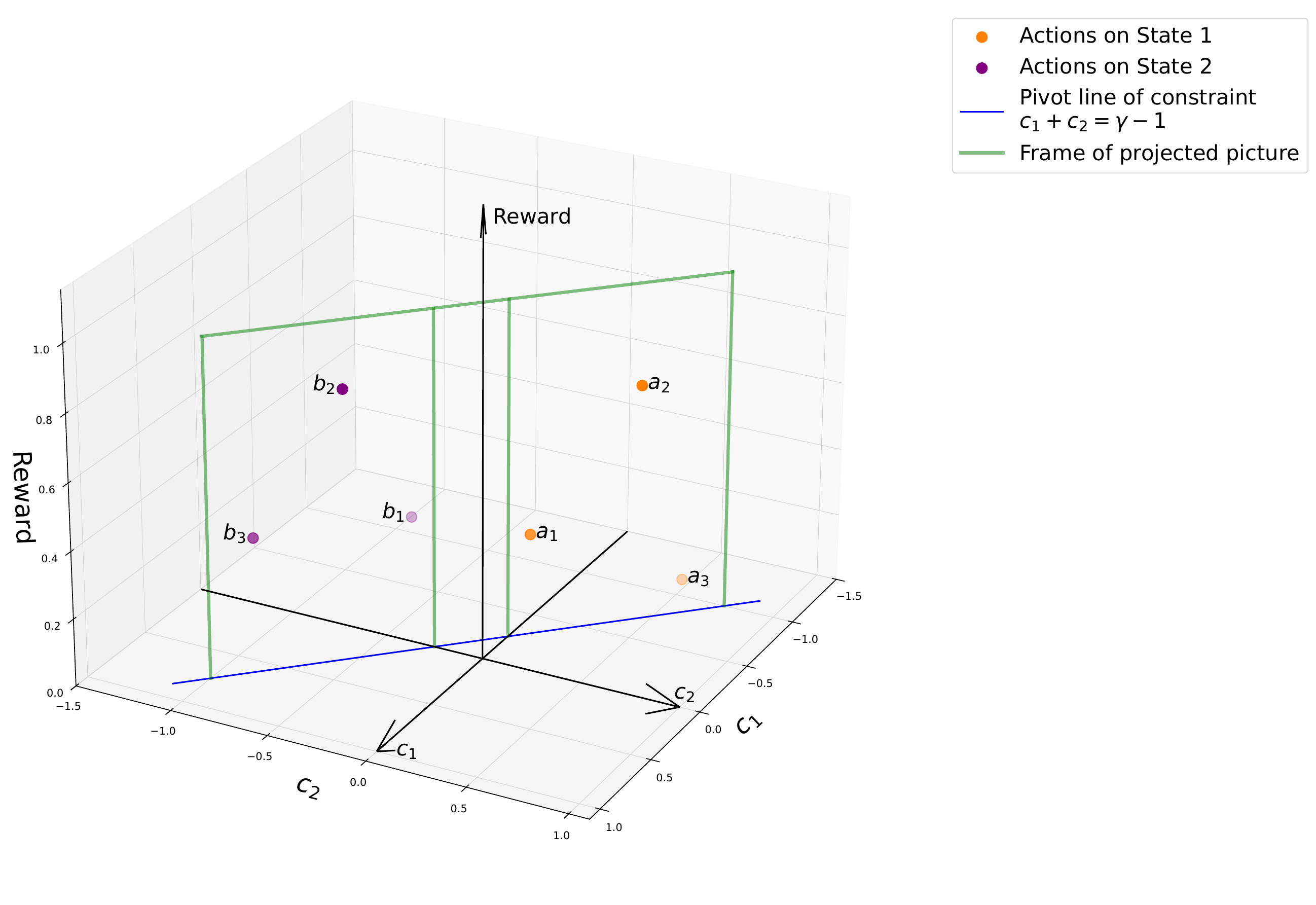}%
\end{center}
\caption{An example of the action space for 2-states MDP with 3 actions on each state described in Appendix \ref{ssec:geom_example}. Vertical axis show action rewards, and two horizontal axes correspond to the first (axis $c_1$) and second (axis $c_2$) coefficients of action vector.}
\label{fig:action_space_example}
\end{figure*}

In this example, we illustrate how to represent a regular MDP geometrically. Consider a 2-state MDP with a discount factor of $\gamma = 0.75$ and 3 actions available in each state:

\begin{itemize}
    \item State 1:
    \begin{itemize}
        \item Action $a_1$: $r^{a_1} = 0.3$, $p_1^{a_1} = 0.9, p_2^{a_1} = 0.1$ 
        \item Action $a_2$: $r^{a_2} = 0.7$, $p_1^{a_2} = 0.4, p_2^{a_2} = 0.6$
        \item Action $a_3$: $r^{a_3} = 0.1$, $p_1^{a_3} = 0.2, p_2^{a_3} = 0.8$
    \end{itemize}
    \item State 2:
    \begin{itemize}
        \item Action $b_1$: $r^{b_1} = 0.4$, $p_1^{b_1} = 0.1, p_2^{b_1} = 0.9$
        \item Action $b_2$: $r^{b_2} = 0.8$, $p_1^{b_2} = 0.4, p_2^{b_2} = 0.6$
        \item Action $b_3$: $r^{b_3} = 0.4$, $p_1^{b_2} = 0.8, p_2^{b_2} = 0.2$
        
    \end{itemize}
\end{itemize}

Each action vector is constructed with the reward of the corresponding action as its 0th coordinate. To compute the other coordinates, we multiply the transition probabilities by the discount factor $\gamma$ and then adjust them by subtracting $1$ from the coordinate correspondent to the state of the action. Specifically, for actions $a_1$, $a_2$, and $a_3$, we subtract 1 from the first coordinate, while for actions $b_1$, $b_2$, and $b_3$, we subtract 1 from the second coordinate. The resulting action vectors are as follows:

State 1: $a^+_1 = (0.3, -0.325, 0.075), a^+_2 = (0.7, -0.7, 0.45), a^+_3 = (0.1, -0.85, 0.6)$.

State 2: $b^+_1 = (0.4, 0.075, -0.325), b^+_2 = (0.8, 0.3,-0.55), b^+_3=(0.4, 0.6,-0.85)$.

These points are visualized in Figure \ref{fig:action_space_example}. Note that all the vectors satisfy the constraint $c_1 + c_2 = \gamma - 1 = -0.25$, which corresponds to the blue line in the figure.

\subsection{Affine equivalence}

In this section we discuss the relation of the transformation $\mathcal{L}_s^\delta$ and standard RL algorithms: value iteration and policy iteration. Choose MDPs $\mathcal{M}$ and MDP $\mathcal{M}'$, such that $\mathcal{M}' = \mathcal{L}_s^\delta \mathcal{M}$ (which means that $\mathcal{L}_s^\delta$ transforms $\mathcal{M}$ to $\mathcal{M}'$). Note that $\mathcal{M}$ and $\mathcal{M}'$ have the same sets of actions and, therefore, the policies are drawn from the same sets of actions. Then, we can make the following observations regarding the dynamics of the policy iteration and value iteration algorithms in MDPs $\mathcal{M}$ and $\mathcal{M}'$:
\begin{itemize}
    \item Each advantage $\adv{a}{\pi}$ remains the same, which leads to the same action choices during the policy improvement steps of the policy iteration algorithm in $\mathcal{M}$ and $\mathcal{M}'$. Therefore, the transformation $\mathcal{L}_s^\delta$ preserves the dynamics of the Policy Iteration algorithm.
    \item Assume the starting values of the Value Iteration algorithm are $V_0$ in $\mathcal{M}$ and $V_0'$ in $\mathcal{M}'$, such that $V_0(s) + \delta = V_0'(s)$. Then, due to preservation of the advantages, the choice of action remains the same at every step of the algorithm. At the same time, since we increase the value of state $s$ by $\delta$ for every pseudo-policy, the difference between the two value vectors remains unchanged. This, in turn, implies that $\mathcal{L}_s^\delta$ preserves the error vectors $e_t = V_t - V^*$, which defines the dynamics of the value iteration algorithm and set of $\epsilon$-optimal policies. It also preserves the values of $V_{t+1} - V_{t}$, the vector used to compute the stopping criteria in the Value Iteration algorithm. All these properties together imply that the transformation $\mathcal{L}_s^\delta$ preserves the dynamics of the Value Iteration algorithm if the initial values are adjusted.
\end{itemize}

The dynamics of the main algorithms imply a similarity between MDPs $\mathcal{M}$ and $\mathcal{M}'$, based on which we give the following definition:

\begin{definition}
MDPs $\mathcal{M}$ and $\mathcal{M}'$ are \textbf{affine equivalent} if there exists a vector $\Delta = (\delta_1, \dots, \delta_n)$ such that
\begin{equation*}
\mathcal{M}' = \mathcal{L}^\Delta\mathcal{M}.
\end{equation*}
Affine equivalence is a transitive relation on the set of MDPs. An \textbf{affine equivalence class} of MDPs is a maximal subset of MDPs in which every two MDPs $\mathcal{M}_1$ and $\mathcal{M}_2$ are affine equivalent.
\end{definition}

We then demonstrate the geometrical dynamics of the main MDP algorithms~---Value Iteration and Policy Iteration~---in the action space. We analyze the dynamics of the normalized MDP since every MDP can be transformed into its normal form while keeping the same affine equivalence class.

\subsection{Policy Iteration} \label{ssec:policy_iteration}

\begin{figure*}[t!] 
\begin{center}
\includegraphics[width=\textwidth]{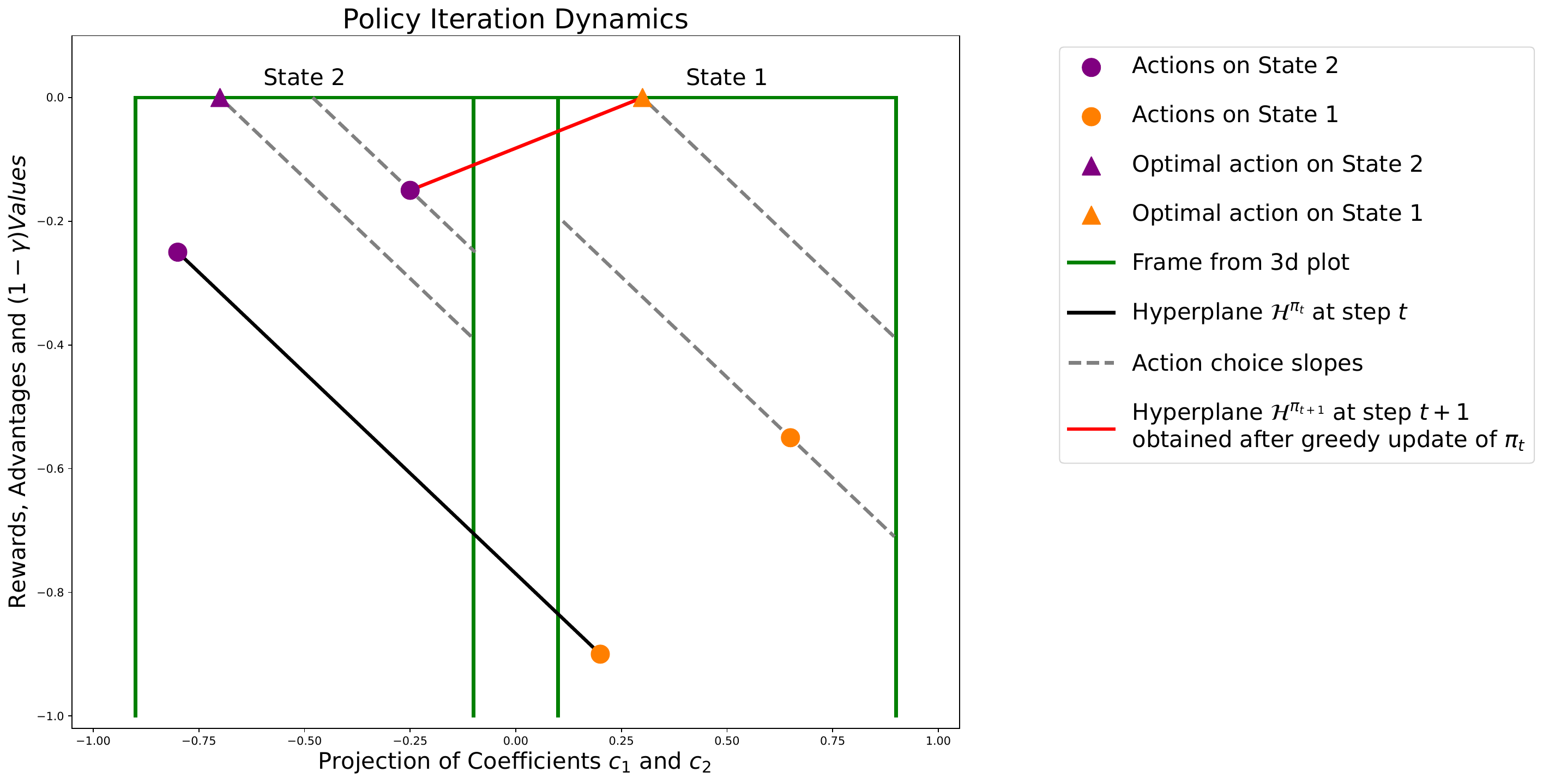}%
\end{center}
\caption{Illustration of the Policy Iteration algorithm dynamics in 2-state MDP.}
\label{fig:policy_iteration_dynamics}
\end{figure*}

The example dynamics of the policy iteration algorithm is presented in Figure \ref{fig:policy_iteration_dynamics}. Policy iteration consists of two steps applied sequentially: policy evaluation and policy improvement. As follows from Proposition \ref{pro:policy_eval}, the policy evaluation step is identical to hyperplane construction. For the policy improvement step, at each state it chooses the highest actions with respect to the current policy $\pi_t$. Therefore, with every iteration, policy iteration algorithm gets a higher hyperplane until the optimal one is reached. 

\subsection{Value Iteration} \label{ssec:vi_dynamics}

\begin{figure*}[t!] 
\begin{center}
\includegraphics[width=\textwidth]{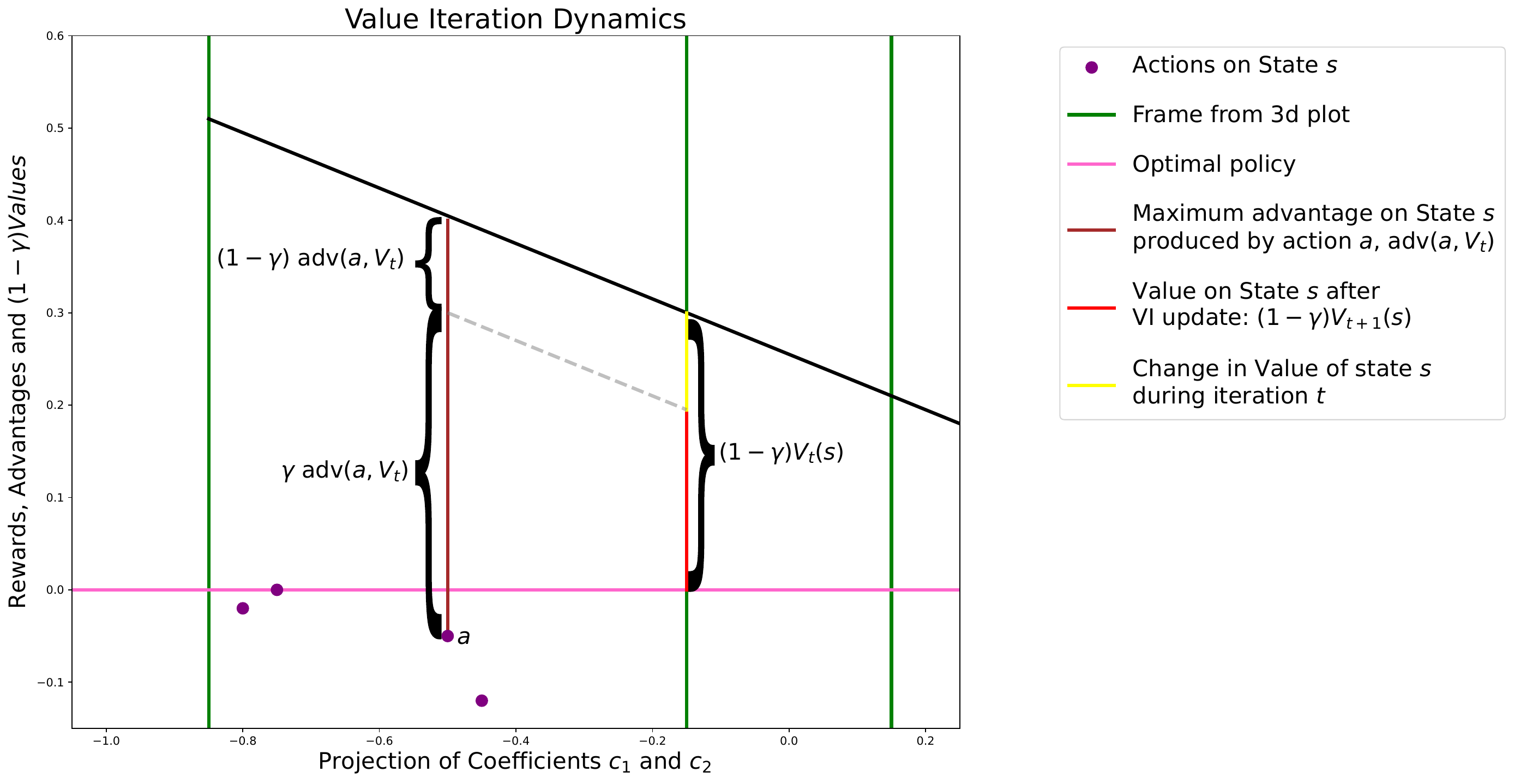}
\end{center}
\caption{Illustration of the Value Iteration algorithm dynamics.}
\label{fig:value_iteration_dynamics}
\end{figure*}

The dynamics of the Value Iteration algorithm are a little bit more complex and are presented in Figure \ref{fig:value_iteration_dynamics}. At each step, value iteration performs the following update:
\begin{align*}
    V_{t+1}(s) &= \max_a \left[ r(s,a) +\gamma \sum_{s'} P(s'|s,a) V_{t}(s') \right] \\
                &= V_{t}(s) + r(s,a^*_t) + (\gamma P(s|s,a^*_t)-1)V_t(s) +\\
                &\gamma \sum_{s' \ne s} P(s'|s,a^*_t) V_{t}(s')= V_t(s) + \adv{a^*_t}{V_{t}},
\end{align*}
where $a^*_t$ is the action maximizing the advantage on state $s$ at timestep $t$.

Therefore, Value Iteration simply finds the highest advantage and adds it to the current state value, which, in normal form, is equal to the state error. Note that, for positive values, advantages are always negative, which guarantees that the Value Iteration update decreases the norm of the values. Graphically, advantages and values are multiplied by $(1-\gamma)$, ensuring that the Value Iteration update is safe: if values were positive before the update, they will remain positive afterward.

\section{THEOREM PROOFS}

\subsection{Proof of Theorem \ref{thm:adv_preservation}} \label{ssec:adv_preservation_proof}
    
We need to prove that the advantage of any action with respect to any policy does not change after the transformation $\mathcal{L}^\delta_s$. The first case is $\mathrm{st}(a) = s$. Denote the original value vector of a policy as $V_+$, the modified value vector as $V'_+$, the original action vector as $a^+$, and the modified action vector as $(a')^+$. Then
\begin{align*}
    a^+V_+ =\, &r^{a} + (\gamma p^a_s - 1)V^s_+ + \gamma \sum_{s' \ne s} p^{s'}_a V^{s'}_+ =\\
    &r^{a} - \delta(\gamma p^a_s - 1) + (\gamma p^a_s - 1)(V^s_+ + \delta)+ \gamma \sum_{s' \ne s} p^{s'}_a V^{s'}_+ = (a')^+ V'_+.
\end{align*}
For action $\mathrm{st}(b) \ne s$ the proof looks essentially the same. Denoting the corresponding action vectors as $b^+$ and $(b')^+$ we have:
\begin{align*}
    b^+V_+ =\, &r^{b} + \gamma p^b_s V^s_+ + \gamma \sum_{s' \ne s} p^{s'}_b V^{s'}_+ = \\
    &r^b - \delta\gamma p^b_s + \gamma p^b_s (V^s_+ + \delta) + \gamma \sum_{s' \ne s} p^{s'}_b V^{s'}_+ = (b')^+ V_+'.
\end{align*}

\subsection{Proof of Theorem \ref{thm:hier_conv}} \label{ssec:hier_conv_proof}

The theorem can be proven by induction on the class number. We are going to show that after $t$ iterations of the algorithm $t$ lowest hierarchy classes have maximum reward of $0$.

\textbf{Basis:} The lowest class of 1~--- leaf nodes. Each of states correspondent to these nodes has only self-loop actions, and after first iteration of the algorithm maximum reward on each node becomes $0$. That does not change during the run of the algorithm, since probability to transfer to other nodes is $0$ and, consequently, $\delta$'s associated with these states remain $0$ during the run of the algorithm, so they won't affect other states.

\textbf{Step:} By induction hypothesis after $t$ steps of the algorithm for each state $s$ of class $c \leq t$, maximum reward on state $s$ is equal to $0$, which implies that $\delta_s$ computed during the $t+1$ steps is also equal to $0$. Now recall that for any action $a,\, \mathrm{st}(a) = s$ its new reward $r^a_{t+1}$ can be expressed as:

\begin{align*}
r^a_{t+1} = &r^a_t - \delta_s (\gamma p^a_s - 1) + \gamma \sum_{s' \ne s} p_a^{s'} \delta_{s'} = \\
&r^a_t - \delta_s (\gamma p^a_s - 1) + \gamma \sum_{s' \in c, c \le t} p_a^{s'} \delta_{s'} + \gamma \sum_{s' \in c, c > t} p_a^{s'} \delta_{s'}\\
\end{align*}

Note that for any action $a$ the last sum term is always $0$. Then, if class $c$ of $s$ is less or equal to $t$, $\delta_s$ and all $\delta_{s'}$ in the second summation term are $0$, which means that reward of action $a$ does not change and, consequently, the maximum reward on state $s$ remains $0$. In the case when $c=t+1$, both summation terms are equal to $0$, and $r^a_{t+1} = r^a_t - \delta_s (\gamma p^a_s - 1)$. However the quantity $\delta_s$ is chosen such that it maximizes the quantity $r^a_t/(\gamma p^a_s - 1)$, which implies that $\max_a r^a_{t+1}=0$ for every state $s$ in class $t+1$. This concludes the proof of the induction step and the theorem.

\subsection{Proof of Theorem \ref{thm:rb_conv} and Corollary \ref{cor:adv_rew_conv}} \label{ssec:rb_conv_proof}

We begin the proof by reiterating and clarifying the notation used. We operate on state-based vectors, meaning all vectors have dimension $n = |\mathcal{S}|$, and matrices are of size $n \times n$.

\begin{itemize}
\item $r_{\max}$~--- One of the key quantities used throughout the paper. It is defined as the minimum over the maxima of action rewards in each state: $r_{\max} = -\min_s \max_{a, \mathrm{st}(a) = s} r^a_0$. Roughly, it can be understood as the (negative of the) distance the algorithm needs to cover during its run.
\item $\square_t$~--- The subscript $t$ refers to quantities obtained after $t$ iterations of the algorithm. This is important in our setting because rewards change at each step.
\item $\square_{\pi_t}$~--- The subscript $\pi_t$ refers to quantities related to the policy $\pi_t$, which is the policy formed by the set of actions chosen during iteration $t$ of the algorithm. We use this as a subscript only for objects related to transition probabilities.
\item $P_\pi$~--- The transition matrix produced by the policy $\pi$.
\item $D_\pi = {\rm diag}(I - \gamma P)$~--- A diagonal matrix consisting of the elements from the diagonal of the matrix $I - \gamma P$.
 \item $r^{\pi}_t$~--- The rewards of the actions that form the policy $\pi$ after $t$ steps of the algorithm. In particular, $r^{\pi_t}_{t+1}$ represents the rewards of actions from the policy $\pi_t$ after $t+1$ steps of the algorithm (one iteration after policy $\pi_t$ was chosen). 
 \item $V_{t+1}^{\pi}$~--- Similar to rewards, these are the values of the policy $\pi$ produced by the rewards of the policy $\pi$ during iteration $t+1$ of the algorithm. 
 \item $\Delta_t$~--- The vector of update values during the $t$ step of the algorithm. The entries of the vector are chosen according to the rule $\delta(i) = - \max_{a:\, \rm{st}(a)=i} \left[\frac{r^a}{1-\gamma p^a_i}\right]$.
\end{itemize}

Now we start proving the theorem with the following lemma.

\begin{lemma} \label{lemma:stochnorm}
Suppose $P$ is a stochastic matrix and $D = {\rm diag}(I - \gamma P)$. Then the matrix $W$ defined as \[ W = I - D^{-1} (I - \gamma P)
\] satisfies 
\[ ||W||_{\infty}  =  \max_{i} \frac{\gamma  - \gamma [P]_{ii}}{1 - \gamma [P]_{ii}} \le \alpha \] 
\end{lemma}

\begin{proof}
Observe that the matrix $W$ has zero diagonal and
$$ 
W_{ij}  = \frac{\gamma [P]_{ij}}{1 - \gamma [P]_{ii}}\  \mbox{when } i \neq j. 
$$

It is thus a matrix with non-negative elements, and we have \begin{eqnarray*}  ||W||_{\infty} & = & \max_{i} \frac{ \gamma \sum_{j} [P]_{ij}}{1 - \gamma [P]_{ii}} = \max_{i}  \frac{\gamma (1 - [P]_{ii})}{1 - \gamma [P]_{ii}} 
\end{eqnarray*}
\end{proof}
We now give the proof of Theorem \ref{thm:rb_conv}
\begin{proof}
From the definition of $\Delta_t$, we have that 
\begin{equation} \label{eq:deltadef} \Delta_{t} = - D_{\pi_t}^{-1} r_t^{\pi_t},\end{equation}
where $D_{\pi_t}$ again is a diagonal matrix of $I - \gamma P^{\pi_t}$ .

By the definition of the update $\mathcal{L}^{\Delta_t}$, for the policy $\pi_t$ (as well as for any other policy), 
\[ V_{t+1}^{\pi_t} = V_t^{\pi_t} + \Delta_t.\] 

Now plugging in Eq. (\ref{eq:deltadef}),
\[ V_{t+1}^{\pi_t} =  V_t^{\pi_t} - D_{\pi_t}^{-1} r_t^{\pi_t}. \] 

Next, using the fact $V^{\pi} = (I-\gamma P^{\pi})^{-1} r^{\pi}$ we obtain
\[ r_{t+1}^{\pi_t} = r_t^{\pi_t} - (I - \gamma P_{\pi_t}) D_{\pi_t}^{-1} r_t^{\pi_t}. \] Now multiplying both sides by $-D_{\pi_t}^{-1}$:
\begin{equation} \label{eq:r1} 
-D_{\pi_t}^{-1} r_{t+1}^{\pi_t} = -D_{\pi_t}^{-1} r_t^{\pi_t} - D_{\pi_t}^{-1} (I - \gamma P_{\pi_t}) (-D_{\pi_t}^{-1} r_t^{\pi_t}) 
\end{equation}
At time $t+1$, we again choose  the next policy $\pi_{t+1}$ implicitly (via choosing the minimizing actions). We thus have the key inequality:
\begin{eqnarray} 
\Delta_{t+1} &  = & -  D_{\pi_{t+1}}^{-1} r_{t+1}^{\pi_{t+1}} \label{eq:deltae} \\ &  
\leq &  - D_{\pi_t}^{-1} r_{t+1}^{\pi_t}, \label{eq:r2}
\end{eqnarray} 
where the last equality is true because it is the actions in $\pi_{t+1}$ that achieve the maxima at time $t+1$. 

We can now put together Eq. (\ref{eq:r1}) and Eq. (\ref{eq:deltae}) and  Eq. (\ref{eq:r2})  to obtain 
\[ - D_{\pi_{t+1}}^{-1} r_{t+1}^{\pi_{t+1}}  \leq  - D_{\pi_t}^{-1} r_t^{\pi_t} - D_{\pi_t}^{-1} (I - \gamma P_{\pi_t}) (- D_{\pi_t}^{-1} r_t^{\pi_t}).\]  
We rewrite this as 
\begin{eqnarray*} 
\Delta_{t+1}  \leq  \Delta_t - D_{\pi_t}^{-1} ( I  - \gamma P_{\pi_t}) \Delta_t =  W_{\pi_t} \Delta_t,
\end{eqnarray*} 
where similarly to before
\[ W_{\pi} := I - D_{\pi}^{-1} (I - \gamma P_{\pi}).\] 
Using induction and Lemma \ref{lemma:stochnorm}, we obtain
\[ ||\Delta_{t}||_{\infty} \leq \alpha^t ||\Delta_0||_{\infty},\]
where $\alpha$ comes from the definition of the theorem. 

We have thus shown that the size of the update to the value function decreases exponentially in the infinity norm. We next translate this into a direct suboptimality bound on $V^{\pi_t}$. 

Starting from 
\[ r_{t}^{\pi_t} = -D_{\pi_t} \Delta_t,\]
and  
\[ V_t^{\pi_t} = (I - \gamma P_{\pi_t})^{-1} r_t,\]
 we  put the last two equations together to obtain
\[ V_t^{\pi_t} = - (I - \gamma P_{\pi_t})^{-1} D_{\pi_t} \Delta_t\]
Let $c = ||(I - \gamma P_{\pi_t})^{-1} D_{\pi_t}||_{\infty}$. Clearly, 
\[ c \leq || ( I - \gamma P_{\pi_t})^{-1}||_{\infty} ||D_{\pi_t}||_{\infty} \leq \frac{1}{1-\gamma} l.\] 

We then have that 
\[ ||V_t^{\pi_t}||_{\infty} \leq c \alpha^t ||\Delta_0||_{\infty}. \] Let us use bound 
\[ ||\Delta_0|| \leq \frac{r_{\max}}{1-\beta},\] so that 
\[ ||V_t^{\pi_t}||_{\infty} \leq c \frac{\alpha^t}{1-\beta} r_{\max}. \] 
The RHS of the above equation is $\epsilon_t$ from the definition of the theorem. Since for any policy $\pi$ its values $V_t^{\pi}$ are always nonpositive (since rewards are nonpositive throughout execution of the RB-S by Proposition \ref{prop:nonpositive}), the values of optimal policy are upper bounded  by $0$, implying that $\pi_t$ is $\epsilon_t$-optimal.

\end{proof}

\begin{proof}[Proof of Corollary \ref{cor:adv_rew_conv}] Let $\pi_t$ be the policy defined earlier achieving the minimum at time $t$ and let $\pi^*$ be an optimal policy. We have that 
\[ 0 \geq V_t^{\pi^*} \geq V_t^{\pi_t} \geq -r_{\max} \frac{\gamma^t}{1-\gamma}.\] Since 
\[ V_t^{\pi^*} = V_0^{\pi^*} + \sum_{k=0}^{t-1} \Delta_k,\]
we obtain that if we define 
\[ d_t = V_0^{\pi} - \left( -\sum_{k=0}^{t-1} \Delta_k\right), \] then 
\begin{equation} \label{eq:dinfinity} || d_t||_{\infty} \leq r_{\max} \frac{\gamma^t}{1-\gamma}. 
\end{equation} 

Next, we have  
\[ r_t = r_0 + \sum_{k=0}^{t-1} (I - \gamma P) \Delta_k. \] We can write this as 
\[ r_t = r_0  + \sum_{k=0}^{t-1} (\gamma P - I) (V_0^{\pi^*} - d_t), \] or 
\[ r_t = r_0 + \gamma P V_0^{\pi^*} - V_0^{\pi^*} - (\gamma P - I) d_t.\]
Finally, for every action $a$, 
\[ r_t^a =\adv{a}{\pi^*} + (\gamma P - I) d_t.\] 
Finally, using the bound $||\gamma P - I||_{\infty} \leq 2$ and Eq.  (\ref{eq:dinfinity}) concludes the proof. 
\end{proof} 

\subsection{Proof of Theorem \ref{thm:stoch_conv}} \label{ssec:stoch_conv_proof}

\setlength{\textfloatsep}{10pt}
\begin{algorithm}[t]
   \caption{Stochastic RB-S}
   \label{alg:stoch_rbs}
\begin{algorithmic}
   \STATE {\bfseries Initialize} MDP $\mathcal{M}$ with rewards $r_1$ and $t=1$.
   \STATE \textbf{Update:} \begin{itemize}
       \item Compute estimation matrix $P_t = \frac{P^1_t + \cdots+P^k_t}{k}$.
       \item Update the MDP rewards $r_{t+1} = r_{t} -R_{t} +\gamma P_t R_{t}$.
       \item Store the maximum action reward in each state as $R_{t+1}$. Set $R^m_{t+1}$ to the minimum of these maxima. 
   \end{itemize} 
      
   \IF{ $|R^m_{t+1}| /(1-\gamma) \ge \epsilon$}
    \STATE Increment $t$ by 1 and return to the update step.
   \ELSE
    \STATE Output policy $\pi$, such that $\pi(s) = \argmax_{a}  r^a_{t+1}$.
    \ENDIF
\end{algorithmic}
\end{algorithm}

 We begin the proof by reiterating and clarifying the notation used. We work with action-based vectors, meaning all vectors have dimension $m = |\mathcal{A}|$, and matrices are of size $m \times m$.

\begin{itemize}
\item $r_t$~--- the vector of all action rewards at step $t$ of the algorithm (assuming the order of actions is fixed).
\item $r^a_t$~--- the reward of action $a$ at step $t$ of the algorithm.
\item $R_t$~--- a vector with entries corresponding to the same actions as $r_t$, but each entry represents the maximum reward available in the state to which the action belongs (at step $t$ of the algorithm). Specifically, $R_t(a) = \max_b r^b, \mathrm{st}(b) = \mathrm{st}(a)$. Additionally, we abuse the notation to define $R_t(s) = \max_b r^b, \mathrm{st}(b) = s$ for state $s$.
\item $R^{\min}_t = \min R_t$~--- the minimum of state maximum rewards. This is a key quantity to track during the run of the algorithm, as it provides a stopping criterion, as shown in Lemma \ref{lem:approxite_solution}.
\item $r_{\max}$~--- the maxima of action rewards in each state: $r_{\max} = -\min_s \max_{a, \mathrm{st}(a) = s} r^a_0$. Roughly, it can be understood as the distance the algorithm needs to cover during its run. In this case, $r_{\max} = -R^{\min}_0 = -\min R_0$.
\item $P_\pi$~--- action-to-action transition matrix with size of $m \times m$. Its element $P_\pi[i,j]$ is the probability that after choosing action $i$, the next action chosen is action $j$. This value is non-zero only if the agent has a non-zero probability of appearing in state $s' = \mathrm{st}(j)$, and policy $\pi$ chooses action $j$ in $s'$. Alternatively, $ P_\pi$ can be seen as a product $P_A C_\pi$, where $P_A$ is a $m\times n$ matrix of action transition probabilities (independent of policy) and $C_\pi$ is an $n \times m$ is policy matrix - $C_\pi[i,j]=1$ if policy $\pi$ chooses action $j$ in state $i$ and $0$ otherwise. 
\end{itemize}

The formal definition of Stochastic RB-S algorithm is presented as Algorithm \ref{alg:stoch_rbs}. The following proposition demonstrates the geometric convergence of the $R^{\min}_t$ quantity during the run of the algorithm.

\begin{proposition} \label{prop:r_gamma_contract} We have that for all $t$, $R^{\min}_{t+1} \geq \gamma R^{\min}_t.$ \label{prop:stochl}
\end{proposition} 

\begin{proof}
Let $s$ be the state where minimum is achieved, $R_{t+1}(s) = R_{t+1}^{\rm min}$. Suppose action $a$ has the maximum reward during iteration $t$, $R_t(s) = r^a_t$. Then,
\begin{align*}
R_{t+1}^{\rm min} \ge r^a_{t+1} = r^a_t - R_t(a) + (\gamma P^t R_t)(a) =
(\gamma P^t R_t)(a) \ge R^{\min}_t,
\end{align*}
where the last inequality follows from stochasticity of $P_t$.
\end{proof}

During the run of the algorithm rewards of all actions are always negative, which is demonstrated in the following proposition, similar to Proposition \ref{prop:nonpositive}

\begin{proposition} \label{prop:nonp}
$r_t \leq 0$ for all $t$. 
\end{proposition} 

 \begin{proof} We will prove this by induction. This holds at time zero by the inductive assumption. Now suppose it holds at time $t$. In particular, $\max r_t \leq 0$. It suffices to prove that $\max r_{t+1} \leq 0$. But:
 \begin{eqnarray*} \max r_{t+1} & = & \max (r_t - R_t + \gamma P_t R_t) \\ 
 & \leq & R_r - R_r + \max \gamma  P_t R_t \\ & = &  \gamma \max P_t R_t.
 \end{eqnarray*} By the inductive hypothesis we have $\max r_t \leq 0$ and combining that with the stochasticity of $P_t$, we obtain the RHS is nonpositive. This completes the proof. 
 \end{proof} 


Next, let us rewrite our update rule to have an explicit expression for the stochastic error at each time step. Instead of Eq. \ref{eq:rb_sampling} let us write 
 \[ r_{t+1} = r_t - R_t + \gamma P_{\pi_t} R_t + z_t,\]
 where
 \[ \E[z_t|r_t, \ldots, r_0]=0.\]
 Specifically, we have 
 \[ z_t = \gamma (P_t - P_{\pi_t}) R_t.\]
 The size of the noise term $z_t$ plays an important role in our analysis. What turns out to matter is the infinity norm of the sum of these random variables, which is bounded in the next proposition.

 \begin{proposition} The random variable 
 \[ Z_t(s) = \sum_{l=0}^{t-1} z_l(s)\]
 is a sub-Gaussian random variable with parameter $\sigma^2 =  \frac{r_{\max}}{k(1-\gamma^2)}$.
 \end{proposition} 
 \begin{proof} 

 Observe that each $z_l(s)$ is the average of $k$ random variables whose support lies in $[-r_{\max}\gamma^l,r_{\max}\gamma^l]$ by  Proposition \ref{prop:r_gamma_contract}. Let us denote these random variables as $e_{t}^i$, so that  
 \[ z_l(s) = \frac{\sum_{i=1}^{k} e_{l}^i}{k}.\] We thus have that 
 \[ P\left(\left|\sum_{l=0}^{t-1} z_l (s)\right| \ge \epsilon \right) = P\left(\left|\sum_{l=0}^{t-1} \sum_{i=1}^{k} \frac{e_l^i (s)}{k} \right| \ge \epsilon\right). \] 
 We can apply Azuma's inequality to the martingale
 $$
 \left\{0, \frac{e^1_0}{k}, \frac{e^1_0 + e^2_0}{k}, \ldots, \frac{\sum_{i=1}^k e^i_0}{k}, \frac{\left(\sum_{i=1}^k e^i_0\right) + e^1_1}{k}, \ldots \right\},
 $$ 
at time $kt$ to yield
 \[ 
 P\left(\left|\sum_{l=0}^{t-1} \sum_{i=1}^{k} \frac{e_l^i}{k} \right| \ge \epsilon\right) < 2\exp{\frac{-2\epsilon^2}{\sum_{l=0}^{t-1} k c_l^2} }.
 \] 
 
 where $c_l$ is the support of  random variables $\frac{e_l^i}{k}$, which is equal to $\frac{r_{\max}\gamma^{l}}{k}$ and 

 \begin{equation*}
 \sum_{l=0}^{t-1} k c_l^2 \le  \sum_{l=0}^{t-1} k c_l^2 = \sum_{l=0}^{t-1} \frac{r_{\max}^2 \gamma^{2l}}{k} \le \frac{r_{\max}^2 }{k(1-\gamma^2)}.
 \end{equation*}
 
 We conclude that $\sum_{l=0}^{t-1} z_l(s)$ is a sub-Gaussian random variable with parameter $\sigma^2 =  \frac{r_{\max}^2}{k(1-\gamma^2)}$. 

 \end{proof} 

Given a policy  $\pi$, let us denote by $G_t^{\pi}$ the vector 
\[ G_{t}^{\pi} = Q_{0}^{\pi} - Q_t^{\pi},\]
where $Q_{t}^{\pi}$ is the Q-value vector of the policy $\pi$ under the rewards $r_t$. This vector has as many entries as the number of actions. Note that because $r_t$ is now a random variable, so is $Q_{t}^{\pi}$ and thus also $G_t^{\pi}$. 

\begin{lemma}  \label{lemm:gbound} 
\[ \E \left( \max_{\pi,\pi'} ||G_t^{\pi} - G_t^{\pi'}||_{\infty} \right)\le \frac{2}{1-\gamma} \E || Z_t||_{\infty}.
\]
\end{lemma} 

\begin{proof} Indeed, for any fixed policy $\pi$, 
\[ Q_{t}^{\pi} = (I - \gamma P_{\pi})^{-1} r_t,\]
where 
$P_{\pi}$ is the action-to-action transition matrix. Thus 
\[ Q_{t+1}^{\pi} = Q_t^{\pi} + (I - \gamma P_{\pi})^{-1} (\gamma P_{\pi_t} - I) R_t + (I-\gamma P_{\pi})^{-1} z_t.
\] 
A key observation is that  the second term on the RHS does not depend on the policy, despite appearing to do so. 
Indeed, matrix $P_{\pi_t} = P_A C_{\pi_t}$ and $P_{\pi_t} = P_A C_{\pi}$, as it described in the beginning of the Section, but since $R_t$ has the same entries on the coordinates correspondent to the action on the same state, therefore $C_{\pi_t}R_t = C_{\pi}R_t$ and $P_{\pi_t}R_t = P_\pi R_t$. Therefore 
\[ P_{\pi_t} R_t = P_\pi R_t,\] for {\em any} policy $\pi$. We thus have 
\[ Q_{t+1}^{\pi} = Q_t^{\pi} - R_t + (I-\gamma P_{\pi})^{-1} z_t.\] 
The independence of $\pi$ and the second term of the RHS is clear here in the above equation. 

Now, iterating this recurrence, we obtain that for any two policies $\pi,\pi'$, 
\begin{equation} \label{eq:gdiffeq} G_t^{\pi} - G_t^{\pi'} = (I - \gamma P_{\pi})^{-1} \sum_{l=0}^{t-1} z_l - (I - \gamma P_{\pi'})^{-1} \sum_{l=0}^{t-1} z_l. \end{equation} Thus 
\[  ||G_t^{\pi} - G_t^{\pi'}||_{\infty} \leq || (I - \gamma P_{\pi})^{-1} - (I - \gamma P_{\pi'})^{-1} ||_{\infty} \times \left\| \sum_{l=0}^{t-1} z_l \right\|_{\infty}. \] Therefore 
\[  || G_t^{\pi} - G_t^{\pi'}||_{\infty} \leq \frac{2}{1-\gamma} \left\|\sum_{l=0}^{t-1} z_l\right\|_{\infty}. \]
Since this bound holds for any policy $\pi$ and any policy $\pi'$, we have in fact proven 
\[ \max_{\pi,\pi'} || G_t^{\pi} - G_t^{\pi'}||_{\infty} \leq \frac{2}{1-\gamma} \left\|\sum_{l=0}^{t-1} z_l\right\|_{\infty}. \]
\end{proof} 

We can now prove our main result of this section. 

\begin{proof}[Proof of Theorem \ref{thm:stoch_conv}] By Proposition \ref{prop:stochl}, at $t=\frac{1}{(1-\gamma)}\log (\frac{r_{\rm max}}{\epsilon (1-\gamma)})$ we have that 
\[ R^{\min}_t \geq \gamma^t R^{\min}_0 \geq - \epsilon (1-\gamma).\]
It follows that $Q_t^{\pi_t} \geq -\epsilon {\bf 1}$. On the other hand, using nonpositivity of rewards from Proposition \ref{prop:nonp}, we obtain that $Q_t^{\pi} \leq 0$ for any policy $\pi$. We conclude that the policy $\pi_t$ is $\epsilon$-optimal under the rewards $r_t$. 

However, what we need to bound is the optimality of $\pi_t$ under the initial rewards $r_0$. 
Observe that for any policy $\pi$ we have by definition of $G_t^{\pi}$, 
\[ Q_t^{\pi} = Q_0^{\pi} + G_t^{\pi}.\]  
By $\epsilon$-optimality of $\pi_t$ under the rewards $r_t$, we have that for any policy $\pi$,  
\[ Q_0^{\pi} +  G_t^{\pi} \leq Q_0^{\pi_t} + G_t^{\pi_t} + \epsilon {\bf 1},\] or 
\begin{equation} \label{eq:gdiffbound} Q_0^{\pi} - Q_0^{\pi_t} \leq G_t^{\pi_t} - G_t^{\pi} + \epsilon {\bf 1}. 
\end{equation}
We conclude that  the event that the policy $\pi_t$ is not $2\epsilon$ optimal under $r_0$ can occur only if for some policy $\pi$, we have 
\[ ||G_t^{\pi} - G_t^{\pi_t}||_{\infty} \geq \epsilon. \] 




Let us call that ``bad'' event $\mathcal{X}$. To prove the theorem we need to upper bound the probability of this event: 
\begin{eqnarray} \label{eq:bound_on_a}
P(\mathcal{X}) & \leq & P (\max_{\pi,\pi'} ||G_t^{\pi} - G_t^{\pi'}||_{\infty}   \geq  \epsilon) \leq 
P \left(\frac{2}{1-\gamma} || Z_t ||_\infty \ge \epsilon \right).
\end{eqnarray}

To proceed we need a high-probability bound on $||Z_t ||_\infty$ where each entry of $Z_t$ is a Sub-Gaussian random variable. While standard, we did not find an explicit statement of what we need to use in the existing literature, so we provide a proof below. 

\begin{proposition} \label{prop:subg_max}
For $m$ random variables $Z_1, \ldots, Z_m$ where $Z_i  \sim \rm{subG}(\sigma^2)$ the following  bound holds \footnote{The authors thank Alexander Zimin for help with the proof.}:

\begin{equation}
P(\max Z_i \geq \varepsilon) \le m \exp{-\frac{\varepsilon^2}{2\sigma^2}}    
\end{equation}
\end{proposition}
\begin{proof}
For any $\lambda > 0$
\begin{eqnarray*}
P(\max Z_i \geq \varepsilon) &\le & P \left(\sum_{i=1}^m \exp{\lambda Z_i} \ge \exp{\varepsilon \lambda} \right)\\
&\leq& \frac{\E \sum_{i=1}^m \exp{\lambda Z_i}}{\exp{\varepsilon \lambda}} \leq
m \exp{\frac{\lambda^2 \sigma^2}{2} - \varepsilon \lambda }, \\
\end{eqnarray*}
where second inequality is Markov inequality and third inequality follows from the properties of subgaussian moment-generating functions.

We next choose the value of $\lambda > 0$ which minimizes the upper bound we have derived, which is $\lambda = \varepsilon/\sigma^2$.  Plugging this value in, we obtain the required result:
\begin{equation*}
P(\max Z_i \geq \varepsilon) \le m \exp{\frac{\lambda^2 \sigma^2}{2} - \varepsilon \lambda } \le m \exp{-\frac{\varepsilon^2}{2\sigma^2}}.
\end{equation*}
\end{proof}
Now applying Proposition \ref{prop:subg_max} (multiplying it by $2$ since we need it for both maximum and minimum) to Equation \ref{eq:bound_on_a} we have:
\begin{equation*}
P(\mathcal{X}) \le P \left(\frac{2}{1-\gamma} 
 || Z_t ||_\infty \ge \epsilon \right)
\le 2 m \exp{-\frac{\epsilon^2 (1-\gamma)^2}{4 \sigma^2 }}.
\end{equation*}
Now, plugging in $\sigma^2 = \frac{r_{\max}^2}{k (1-\gamma^2)} $ we have:
\begin{equation*}
P(\mathcal{X}) \le 2m \exp{-\frac{k (1-\gamma^2)\epsilon^2 (1-\gamma)^2}{4r_{\max}^2}}.
\end{equation*}
Choosing
\begin{equation*}
k = \frac{4r_{\max}^2 \log \frac{2m}{1-\tau}}{\epsilon^2(1-\gamma)^3 (1+\gamma)},
\end{equation*}
guarantees that $P(\overline{\mathcal{X}})\ge \tau$, which concludes the proof.


\end{proof} 

\section{ADDITIONAL ALGORITHMS}

\subsection{Exact Reward Balancing Algorithm} \label{ssec:exact_rb_solver}

In this section, we provide a formal definition of the Exact Reward Balancing algorithm and demonstrate its equivalence to Policy Iteration. The formal definition is presented below (Algorithm ~\ref{alg:pi_analog}).

The issue with the algorithm arises in its \textbf{Update} step: choosing $\Delta_t$ such that the value function of the policy $\pi_t$ in the updated MDP $\mathcal{M}_t$ is zero. This means we must compute $\Delta_t$ values that ensure the new value function of $\pi_t$ is zero, implying they must be exactly the negation of the current value function of $\pi_t$: $\Delta_t = -V^{\pi_t}$.

If we define the transition probabilities under policy $\pi_t$ as $P_{\pi_t}$ and the rewards at iteration $t$ as $r_t^{\pi_t}$, then: 
\begin{equation} 0 = r_t^{\pi_t} + (\gamma P_{\pi_t} - I) \Delta_t \implies \Delta_t = -(\gamma P_{\pi_t} - I)^{-1} r_t^{\pi_t}. \end{equation}
Thus, computing $\Delta_t$ requires inverting an $n \times n$ matrix, which is analogous to the policy evaluation step in the Policy Iteration algorithm.

Next, to compute the update $\mathcal{M}_t = \mathcal{L}^{\Delta_t} \mathcal{M}_{t-1}$, we calculate $r^{a}{t+1} = r^{a}{t} + \sum_{i=1}^{n} c_a(i) \Delta_t(i)$ and select the action with the maximum reward. This step is identical to the policy improvement step.

By combining these two observations, we conclude that the Exact Reward Balancing algorithm is equivalent to Policy Iteration.

\begin{algorithm}[h]
   \caption{Exact Reward Balancing}
   \label{alg:pi_analog}
\begin{algorithmic}
   \STATE {\bfseries Initialize} $\mathcal{M}_0 := \mathcal{M}$ and $t=1$.
   \WHILE{True} 
   \STATE Select policy $\pi_t$ by choosing $a_s = \argmax_a r^{a}, \mathrm{st}(a)=s$ for each state $s$. Denote the reward of the action as $r^{a_s}$.
   \IF{  $r^{a_s}=0 \quad \forall s$}
    \STATE Output policy $\pi_t$ 
   \ELSE
    \STATE \textbf{Update}  the MDP $\mathcal{M}_t = \mathcal{L}^{\Delta_t} \mathcal{M}_{t-1}$ by choosing $\Delta_t$ such that the values of the policy $\pi_t$ in the updated MDP $\mathcal{M}_t$ are all $0$.
    \STATE Increment $t$ by 1 and return to the policy choice step. 
    \ENDIF
    \ENDWHILE 
\end{algorithmic}
\end{algorithm}

\subsection{Intuition behind Safe Reward Balancing algorithm} \label{ssec:rb_intuition}

Having introduced the concept of reward balancing, there are multiple ways to design an algorithm based on this idea. In this section, we provide intuition behind the Safe Reward Balancing algorithm (RB-S), the version presented and analyzed in this paper. The algorithm is considered "safe" in the following sense: action rewards, which start as negative after the initial subtraction of the maximum reward, remain negative throughout the execution of the algorithm. This is guaranteed by the choice of deltas $\delta_s = -\max_a \frac{r^a}{1 - \gamma p^a_s}$, which is justified from both classic MDP and geometric perspectives.

From the classic MDP point of view, for every state $s$, the value of the optimal policy at state $s$ is guaranteed to be at least the maximum over all actions of the sum of discounted rewards collected before leaving state $s$. This value is exactly the value of $\delta_i$ proposed in the algorithm: 
\begin{equation*} \delta_i = -\max_{a \mid {\rm st}(a) = i} \frac{r^a}{1 - \gamma p_a^i}. \end{equation*}
Including action $a^*$, which maximizes this quantity, in a policy $\pi$ guarantees that $V^\pi(i) \le \delta_i$, which is also true for the optimal policy. Therefore, the value of the optimal policy will remain negative, and the choice of $\delta_i$ is safe. A stronger fact is shown in the following proposition.

\begin{proposition} $r_t^a \leq 0$ for all times $t$ and actions $a$. \label{prop:nonpositive}
\end{proposition} 

\begin{proof} Recall that we assumed this to hold at step zero. Now we will argue that this property is preserved by applying each map $L_i^{\delta_i}$ computed in the execution of the RB-S. Indeed, in applying this map have 
\[ \delta_i =  -\max_{a | {\rm st}(a) = i} \frac{r^a}{ 1- \gamma p_a^i }. \] Since  the numerator is negative and denominator is positive, we have 
\[ \delta_i \geq 0.\] This immediately implies that for actions  $b$ with ${\rm st}(b) \neq i$, the update performed by $\mathcal{L}_i^{\delta_i}$,
\[ r^b \leftarrow r^b - \delta_i \gamma p_b^i\] leaves $r_b$ non-positive. 

As for actions $a$ with ${\rm st}(a)=i$, we have that the update has the property that 
\begin{eqnarray*} 
r^a & \leftarrow &  r^a - \delta_i ( \gamma p_s^a - 1) \\ 
& = & r^a + \delta_i ( 1 - \gamma p_s^a) \\ 
& \leq  &  r^a + \frac{r^a}{\gamma p_{s}^a - 1} (1 - \gamma p_s^a)  \\ 
& = &  0.
\end{eqnarray*} 
\end{proof} 

The choice of $\delta_i$ is also justified from a geometric perspective (Figure \ref{fig:rbs_action_choice}). The selection of $\delta_s$ can be interpreted as the choice of a hyperplane $\mathcal{H}^{\delta_s}$, which is incident to one of the actions in state $s$ and zeros on the other self-loop vertical lines $L_{s'}$, where $s' \ne s$. After applying the transformation $\mathcal{L}^{\delta_s}_s$, each action retains its advantage with respect to $\mathcal{H}^{\delta_s}$, or its relative position to it. Therefore, since none of the actions were above the hyperplane before the transformation, all of them remain below or on the hyperplane afterward.

\begin{figure*}[t!] 
\begin{center}
\includegraphics[width=\textwidth]{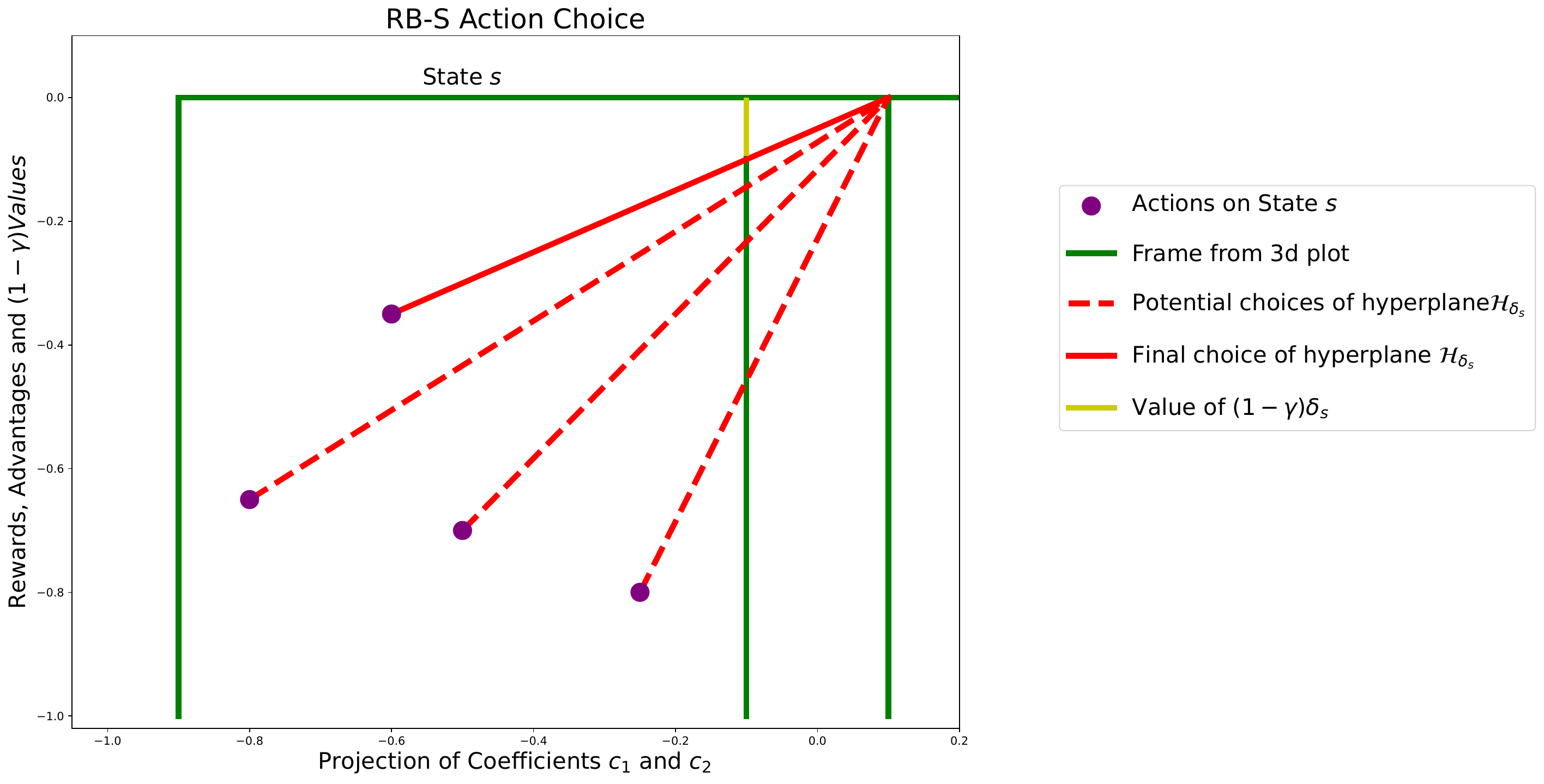}
\end{center}
\caption{RB-S algorithm choice of $\delta_s$. }
\label{fig:rbs_action_choice}
\end{figure*}

\subsection{Reward Balancing with Action Filtering} \label{ssec:rb_w_af}

\begin{algorithm}[t]
   \caption{RB-S with Action Filtering}
   \label{alg:rbs_with_action_filtering}
\begin{algorithmic}
   \STATE {\bfseries Initialize} $\mathcal{M}_0 := \mathcal{M}$, $\mathcal{A}_0 = \mathcal{A}$ and $t=1$.
   \STATE \textbf{Update:} \begin{itemize}
       \item $\forall s$ compute $\delta_s = - \max_a r^a / (1 - \gamma p^a_s)$.
       \item Update the MDP $\mathcal{M}_t = \mathcal{L}^{\Delta_t} \mathcal{M}_{t-1}$, where $\Delta_t = (\delta_1, \dots, \delta_n)^T$.
       \item Update $\mathcal{A}_t$  by filtering actions with criteria $r^a_t < 2r_{\rm max} \frac{\gamma^t}{1-\gamma}$. 
   \end{itemize} 
      
   \IF{$|\mathcal{A}|> n$}
    \STATE Increment $t$ by 1 and return to the update step.
   \ELSE
    \STATE Output policy $\pi$ that sonsits of remaining actions.
    \ENDIF
\end{algorithmic}
\end{algorithm}

Action filtering is a technique that allows an algorithm originally designed as an approximate solver to produce an exact solution under the assumption that the optimal policy is unique. In this section, we describe the algorithm and prove its convergence.

The intuition behind the algorithm is based on Corollary~ ~\ref{cor:adv_rew_conv}, which establishes that:
\begin{equation*}
\left| r_t^a - \adv{a}{\pi^*} \right|  \leq 2 r_{\max}\frac{\gamma^t}{1-\gamma}. 
\end{equation*}
This implies that as the algorithm progresses, the reward of each action converges to its advantage with respect to the optimal policy. Consequently, the rewards of the actions participating in the optimal policy converge to zero, while the rewards of the actions not participating converge to some value less than zero. More precisely, if after $t$ iterations of the RB-S algorithm the reward $r_t^a$ of action $a$ satisfies
\begin{equation} \label{eq:adv_conv_repeat}
r^a_t < 2r_{\rm max} \frac{\gamma^t}{1-\gamma},
\end{equation}
then action $a$ is not optimal and does not need to be considered in the following iterations of the algorithm.

At the same time, Corollary \ref{cor:adv_rew_conv} and Inequality \ref{eq:adv_conv_repeat} guarantee that this moment will eventually occur. Let us denote the true advantage of action $a$ with respect to the optimal policy as $h$, where $h = -\adv{\pi_0}{a} > 0$. Then, after a number of algorithm iterations exceeding $t_h > \log_{\gamma} \left( \frac{(1 - \gamma) h}{4 r_{\rm max}} \right)$, the action $a$ can be filtered out because after $t_h$ iterations it is true that:
\begin{equation*} \left| r_{t_h}^a - \adv{\pi_0}{a} \right| \leq \frac{h}{2}, \end{equation*}
which implies that the observed action's reward satisfies $r_{t_h}^a < \frac{h}{2}$ and can indeed be filtered out.

The two propositions above imply that action filtering will converge after 
\begin{equation*} T_{h}^{\rm max} = \log_{\gamma} \left( \frac{(1 - \gamma) h^{\rm max}}{4 r_{\rm max}} \right), \end{equation*} 
iterations, where $h^{\rm max} = -\max_a \adv^{\pi^*}{a}$ is the negation of the maximum advantage with respect to the optimal policy among non-optimal actions.

Action filtering provides an idea of how to run the algorithm in a more practical setting. Choose the number of iterations $T$ such that the run time is satisfactory. Then, the algorithm will either converge for $t < T$ to the optimal policy by filtering out all non-optimal actions or output an approximate solution after $T$ iterations. Additionally, action filtering during the execution of the algorithm reduces the number of computations required to perform updates.

\subsection{Similarity Between RB-S and VI} \label{ssec:diag_free_sim}

For known MDPs in the diagonal-free case the two algorithms are indeed equivalent. This case corresponds to $0$ self-loop probabilities on every action. In this case, the RB-S simply chooses maximum reward as $\delta_s$ at every iteration and the update of action $a$ becomes:

\begin{equation*}
r^a_{t+1} = r_t^a - R_t(s) + (\gamma P_a R_t)(s),
\end{equation*}
where $R_t$ is the maxima vector defined in Appendix \ref{ssec:stoch_conv_proof}. Note that we can make these choices and updates even in the case when self-loop probabilities are not $0$. 

Let's denote the cumulative adjustment of the RB-S algorithm up to step $t$ as $\tilde{\Delta}_t$:

$$ \tilde{\Delta}_t = \sum_{i=1}^t \Delta_i $$.

Then, we are going to prove that $\tilde{\Delta}_t = V_t$ in the diagonal-free case by induction, given that initial values $V_0 = 0$. Induction base $0 = 0$ is trivial. Induction step can be derived as follows: at step $t$ of the RB-S algorithm, reward $r^a_t$ of action $a$ on state $s$ is equal to:

$$ r^a_t = r^a_0 + C^a \tilde{\Delta}_t = r^a_0 - \tilde{\Delta}_t(s) + \gamma p^a \tilde{\Delta}_t $$

Maximizing this quantity over actions $a$ on the state $s$ is the same as maximizing:

$$ \argmax_a r^a_0 - \tilde{\Delta}_t(s) + \gamma p^a \tilde{\Delta}_t = \argmax_a r^a_0 + \gamma p^a \tilde{\Delta}_t = r^a + \gamma p^a V_t ,$$

by inductive assumption. Therefore, the choice of action will be the same. Now, abusing notation let's assume that action $a$ indeed maximizes the the expression. Then,

$$ V_{t+1}(s) = r^a + \gamma p^a V_t$$

for $t$th step of the VI algorithm. For RB-S, 
$$\Delta_{t+1}(s) = \argmax_a r^a_0 - \tilde{\Delta}_t(s) + \gamma p^a \tilde{\Delta}_t$$
and
$$ \tilde{\Delta}_{t+1}(s) = \Delta_{t+1}(s) + \tilde{\Delta}_t(s) = r^a_0 - \tilde{\Delta}_t(s) + \gamma p^a \tilde{\Delta}_t + \tilde{\Delta}_t(s) =  V_{t+1}(s)$$

Therefore, in the known diagonal-free MDP case, the two algorithms are equivalent. At the same time, in the presence of self-loops, RB-S can also be viewed as a VI algorithm, but run on a modified MDP in which self-loops are removed. To perform this modification, each action $a: \mathrm{st}(a) = s$, needs to be transformed into a self-loop-free action $\bar{a}$ such that it produces the same values for every policy it participates in:
\begin{align*}
0 &= a^+ V^\pi_+ = r^a + (\gamma p_s^i - 1)V^\pi(s) + \gamma \sum_{i\ne 1} p_i^a V^\pi(i) \\
& = \frac{r^a}{1 - \gamma p_s^i} + (-1)V^\pi(s) + \frac{\gamma (1-p_s^i)}{1 - \gamma p_s^i} \sum_{i\ne 1} \frac{p_i^a}{1-p_s^i} V^\pi(i) = \bar{a}^+ V^\pi_+
\end{align*}
Note that in the resulting MDP, each action has a different discount factor $\gamma$ depending on the self-loop probability of the action $a$ before the transformation. Therefore, the standard VI convergence proof is not directly applicable here. However, it is not difficult to modify the proof to show that VI run on this modified MDP will have a convergence rate of at least the maximum of $\frac{p_i^a}{1 - p_s^i}$, which matches the convergence rate stated in Theorem \ref{thm:rb_conv}.

\subsection{Comparison of RB-S and Value Iteration/Q-learning in stochastic case} \label{ssec:rbs_vi_stoch_diff}

Does the equivalence between RB-S and VI, as presented in Section \ref{ssec:diag_free_sim}
, apply to the stochastic case? 
In this Section, we show the difference between RB-S and value-based algorithms. First, we show the difference between RB-S and Stochastic Value Iteration. In the known MDP case the VI update can be written as:

$$ V_{t+1} = \max_a r^a + \gamma P V_t.$$

Subtracting $V_t$ from both sides, we obtain:

$$ V_{t+1} - V_t = \max_a r^a + (\gamma P - I) V_t. $$

Note, that the expression in the right-hand side is indeed the expression of the new reward of the modified MDP $\mathcal{L}^\Delta$, where $\Delta = - V_t$. Then, if we assume $V_0=0$, the expression can be rewritten as:

\begin{equation*} 
 V_{t+1} - V_t = \max_a r^a + (\gamma P - I) \sum_{i=1}^t (V_i - V_{i-1}) = 
\max_a r^a + (\gamma P - I) \sum_{i=1}^t R_i = R_{t+1}   
\end{equation*} 

This equality again demonstrates the equivalence between RB-S and Value Iteration in the deterministic case. But are two algorithms equivalent in the stochastic case? No, they are not. Assuming the matrix $P_t$ is sampled at step $t$ of the algorithm. Then, the difference between subsequent values can be written similarly as in the previous expression:

\begin{equation} \label{eq:VI_dynamics}
V_{t+1} - V_t = \max_a r^a + (\gamma P_t - I) \sum_{i=1}^t (V_i - V_{i-1}).
\end{equation} 

As for RB-S update values, they are equal to:

\begin{align}
R_{t+1} &= \max_a r_{t+1} = \max_a r_{t} + (\gamma P_t - I)R_{t} \nonumber \\
&= \max_a r_{t-1} + (\gamma P_{t-1} - I) R_{t-1} + (\gamma P_t - I)R_{t} \nonumber \\
&=\max_a r^a + \sum_{i=1}^t (\gamma P_i - I)R_{i} \label{eq:rbs_dynamics}
\end{align}

Equations \ref{eq:VI_dynamics} and \ref{eq:rbs_dynamics} shows the difference between two algorithms: the Value Iteration update is more influenced by the stochasticity of the matrix $P_t$, while in RB-S the term $(\gamma P_t - I) R_t$ dependent on estimation obtained at time step $t$ converges to $0$ as $t \rightarrow \infty$.

Then, we show that stochastic RB-S is not equivalent to Q-learning with any learning rate schedule $\alpha_t$. We define the operator $\widetilde{\max\limits_a}$ to replace a value with the maximum on the correspondent state, and $Q^m = \widetilde{\max\limits_a} Q$. Then, for Q-learning with initial values of $0$, the updates are:
\begin{align*}
&Q_0 = 0,\ Q_0^m = 0 \\
&Q_1 = \alpha_1 r,\ Q_1^m = \alpha_1 R \\
&Q_2 = (1-\alpha_2)Q_1 + \alpha_2(r + \gamma P_1 Q_1^m), \quad Q_2^m = \widetilde{\max_a} \left[ (1-\alpha_2)Q_1 + \alpha_2(r + \gamma P_1 Q_1^m) \right].
\end{align*}

Now the equivalence between RB-S and Q-learning can hold if and only if $\alpha_1 = \alpha_2 = 0$. Then, for the third update:
$$ Q_3 = (1-\alpha_3)Q_2 +  \alpha_3 (r + \gamma P_2 Q_2^m), \quad Q_3^m = \widetilde{\max_a} \left[ (1-\alpha_3)Q_2 +  \alpha_3 (r + \gamma P_2 Q_2^m)\right].$$
Therefore, for Q-learning to be equivalent to RB-S, we need to analyze the difference:
\begin{align*}
Q_3^m - Q_2^m &= \widetilde{\max_a} \left[ (1-\alpha_3)Q_2 +  \alpha_3 (r + \gamma P_2 Q_2^m) - Q_2^m \right] \\
&= \widetilde{\max_a}  \left[ (1-\alpha_3)(r + \gamma P_1 Q_1^m) + \alpha_3 (r + \gamma P_2 Q_2^m) - Q_2^m \right] = \\
&= \widetilde{\max_a}  \left[ r - Q_2^m + \gamma P_1 (1-\alpha_3) Q_1^m + \gamma P_2 \alpha_3 Q_2^m \right]. \\
\end{align*}
For the equivalence with the RB-S to hold, this expression should be equal to:
\begin{align*}
\Delta_3 = R_3 = &\widetilde{\max_a}  \left[ r+ (\gamma P_1 - I) Q_1^m + (\gamma P_2 - I)(Q_2^m - Q_1^m)\right] = \\
& \widetilde{\max_a}  \left[ r - Q_2^m + \gamma P_1 Q_1^m + \gamma P_2 (Q_2^m - Q_1^m) \right],
\end{align*}
which is not possible for any value of $\alpha_3$.

This results in an important difference in stochastic RB-S and Q-learning dynamics:
the Bellman update in Q-learning has only one stationary point, while the RB-S update has a subspace of stationary points. This leads to different algorithm dynamics: Q-learning can recover from errors, but progress after a single update might sometimes be negative. In contrast, RB-S cannot recover from errors, but steady progress toward stationary subspace is guaranteed in every step. Therefore, if the results of a Q-learning run are not satisfactory, the output Q-values can be reused to continue running the algorithm. In RB-S, however, the output is not reusable, so the algorithm must be restarted from scratch.

\section{EXPERIMENTAL RESULTS} 
\subsection{Known MDP experiments} \label{ssec:experiments_known} 
In this section, we present the performance results of the RB-S Algorithm (Algorithm ~\ref{alg:simple_vf_alg}) on three experiments. We compare the convergence of the RB-S algorithm with the Value Iteration algorithm.

In all the MDPs below, we construct the action transition probabilities as follows: the number of actions at each state is equal to the number of potential destinations, and each action has a distinct destination assigned to it. Then, the MDP has global parameters:

\begin{itemize} \item \textbf{Execution probability}: the probability that, after taking an action, the agent appears in the destination state assigned to that action. \item \textbf{Random probability}: the probability that, after taking an action, the agent appears in one of the destinations assigned to the state where the action is taken, with the destination chosen according to a unique exponential distribution. \item \textbf{Self-loop probability}: the probability that, despite taking an action, the agent remains in the same state. \end{itemize}
\
The comparison is performed on three kinds of MDPs:

\begin{itemize} \item \textbf{Random MDP:} We create this MDP by choosing, for each state, between 1 and 4 potential transition destinations selected randomly. We then create the same number of actions as there are potential destinations. The reward of each action is the sum of a randomly sampled state reward from the interval $(0, 3)$ and a random action reward from the interval $(-0.5, 0.5)$. The average shortest path in this MDP is $\mathcal{O}(\log n)$.

\item \textbf{Grid World MDP:} One hundred states correspond to the cells in a $10 \times 10$ grid. Possible destinations are "Up," "Left," "Down," and "Right," which are available at any state unless they would move the agent beyond the grid's border. The reward for each action is equal to $0.1$ times the sum of the cell's coordinates, plus a small random component. The average shortest path in this MDP is $\mathcal{O}(\sqrt{n})$.

\item \textbf{Cycle MDP:} Numbered states organized into a cycle. Each state has three available destinations: the first, second, and third states ahead in the cycle, with an equal number of actions. The reward is equal to the state's number multiplied by $0.1$, plus a small random component. The average shortest path in this MDP is $\mathcal{O}(n)$. \end{itemize}

Our results are presented in Figures \ref{fig:exec_prob_ex} \ref{fig:random_prob_exp},  \ref{fig:gamma_exp}, \ref{fig:size_exp}. Experimental details may be found in the figure captions. Overall, the experimental results support the theoretical findings: the RB-S algorithm largely benefits from self-loops, while the VI algorithm benefits from information exchange between nodes.

\begin{figure*}[t!] 
\begin{center}
\includegraphics[width=\textwidth]{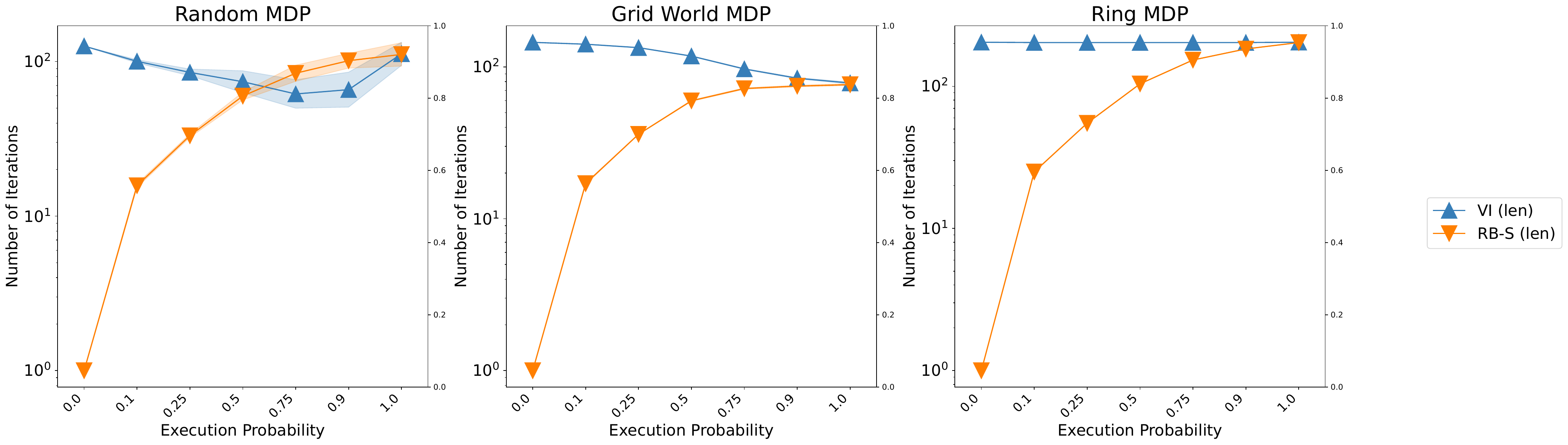}
\end{center}
\caption{Comparison of the performance of VI and RB-S algorithms. X-axis is value of execution probability (while remaining probability is assigned to self-loop). Y-axis is number of iteration it took the algorithm to converge to an $\epsilon$-optimal policy, where $\epsilon=0.1$. Average numbers and log standard deviations are presented. Number of states $n=100$, discount factor $\gamma=0.95$. 3 presented plots demonstrate how RB-S benefits from high probability of self-loops.}
\label{fig:exec_prob_ex}
\end{figure*}

\begin{figure*}[t!] 
\begin{center}
\includegraphics[width=\textwidth]{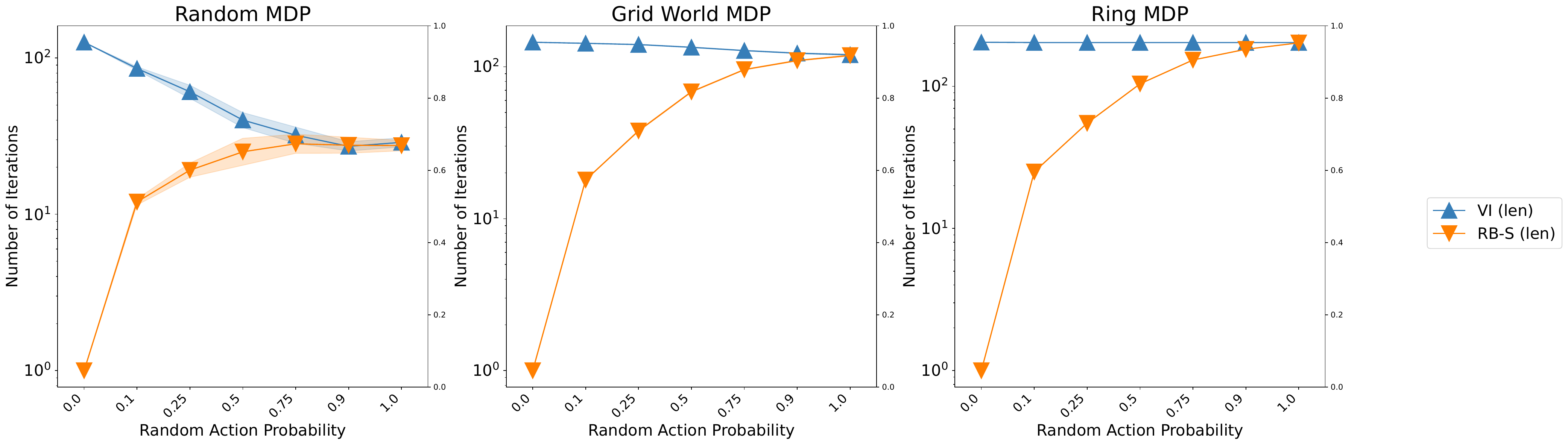}
\end{center}
\caption{Comparison of the performance of VI and RB-S algorithms. X-axis is value of random action probability (while remaining probability is assigned to self-loop). Y-axis is number of iteration it took the algorithm to converge to an $\epsilon$-optimal policy, where $\epsilon=0.1$. Average numbers and log standard deviations are presented. Number of states $n=100$, discount factor $\gamma=0.95$. 3 presented plots demonstrate how RB-S benefits from high probability of self-loops.}
\label{fig:random_prob_exp}
\end{figure*}


\begin{figure*}[t!] 
\begin{center}
\includegraphics[width=\textwidth]{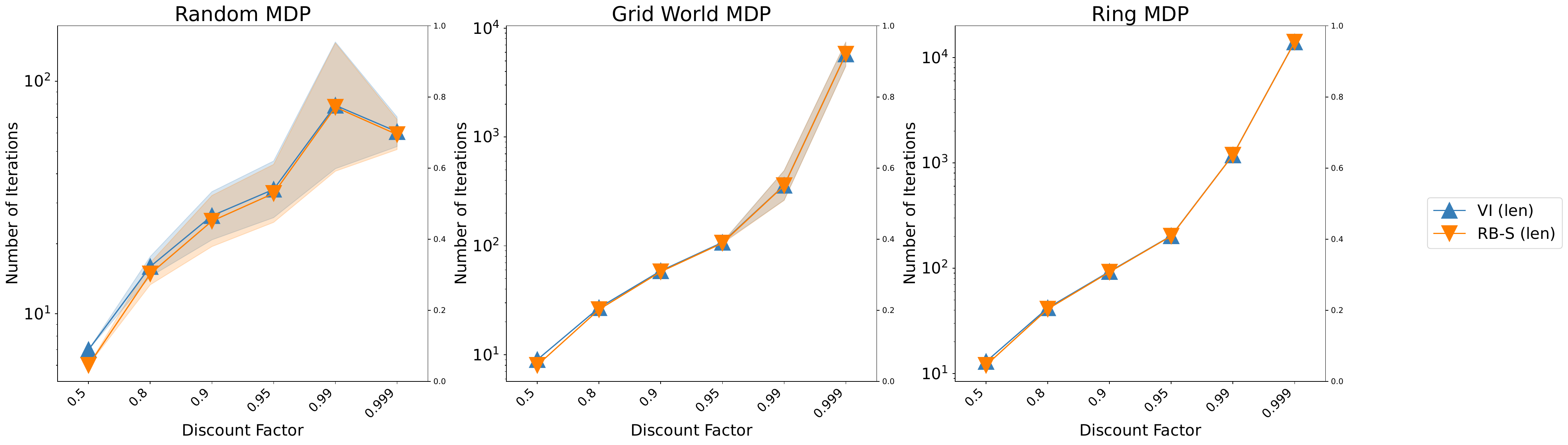}
\end{center}
\caption{Comparison of the performance of VI and RB-S algorithms. X-axis is value of gamma. Y-axis is number of iteration it took the algorithm to converge to an $\epsilon$-optimal policy, where $\epsilon=0.1$. Average numbers and log standard deviations are presented. Number of states $n=100$, execution and random action probabilities are both $0.5$. Performances of RB-S and VI are indistinguishable.}
\label{fig:gamma_exp}
\end{figure*}

\begin{figure*}[t!] 
\begin{center}
\includegraphics[width=\textwidth]{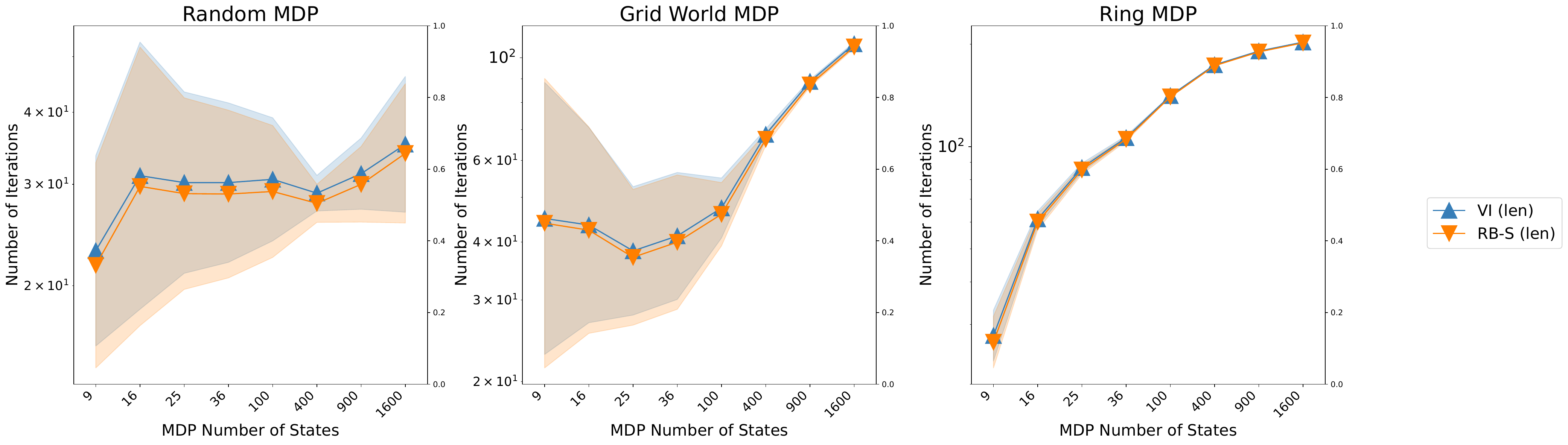}
\end{center}
\caption{Comparison of the performance of VI and RB-S algorithms. X-axis is a number of states. Y-axis is number of iteration it took algorithm to converge to $\epsilon$-optimal policy, where $\epsilon=0.1$. Average numbers and log standard deviations are presented. Execution and random action probabilities are both $0.5$, discount factor $\gamma=0.95$. Performances of RB-S and VI are very similar.}
\label{fig:size_exp}
\end{figure*}

\subsection{Unknown MDP experiments} \label{ssec:unknown_mdp_exp}

As discussed in Appendix \ref{ssec:rbs_vi_stoch_diff}, it is challenging to compare the performance of the two algorithms due to differences between them. In particular, tracking the progress of Q-learning and determining the number of iterations required for convergence is difficult. In contrast, for RB-S, the progress in terms of the infinity norm of the vector $R$ is known, with the error being the unknown factor. In this work, we design the experiments as follows: the number of iterations $t$ is set to
$$\frac{1}{1-\gamma} \log \left( \frac{r_{\max}}{\epsilon (1-\gamma)} \right),$$
which is the number required by RB-S to achieve an $\epsilon$-optimal policy (Theorem \ref{thm:stoch_conv}). We set $\epsilon$ to $0.1$ in all experiments. The number of samples $k$ used to estimate $P$ during each iteration varies and is shown on the $x$ axis. To measure performance, we use the infinity norm between the original rewards of actions of the optimal policy and the action currently implied by the current Q-value or RB-S reward vectors. We use the same MDP types—Random, Grid World, and Ring—as described in the previous section. Specific MDP setups are detailed in the corresponding figures. We run 20 experiments for each setup and plot the mean and standard deviation.

Results are presented on Figures \ref{fig:stoch_exp1}, \ref{fig:stoch_exp2}, \ref{fig:stoch_exp3}. Two algorithms perform similarly in Random MDP, where mixing is quick and Q-Values of all actions are similar, while in Grid World and Ring MDPs performance of RB-S is better.

\begin{figure*}[t!] 
\begin{center}
\includegraphics[width=\textwidth]{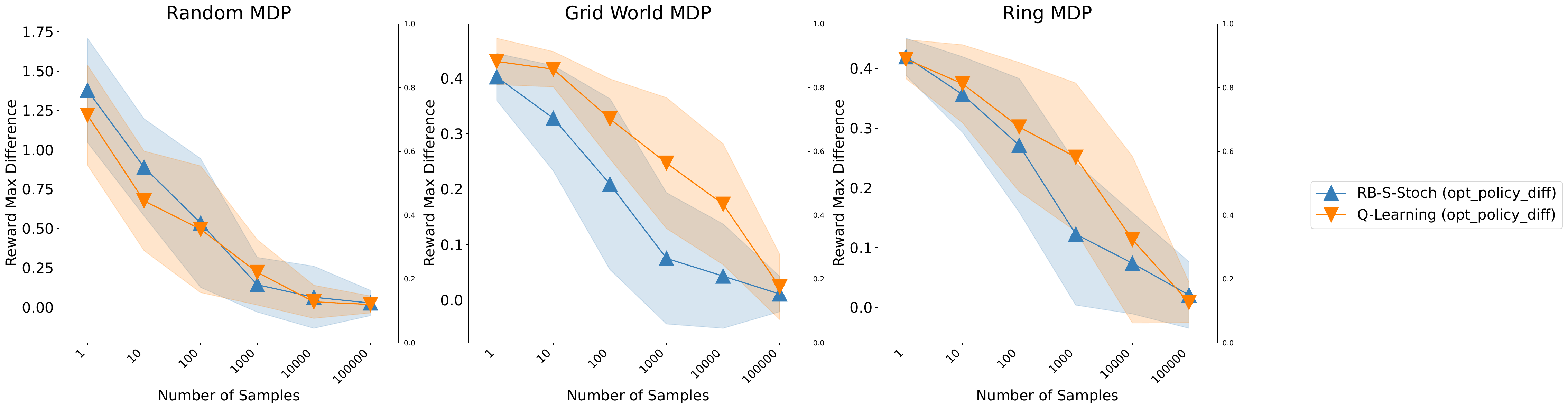}
\end{center}
\caption{Comparison of the performance of Q-learning and stochastic RB-S algorithms. X-axis is a number of samples used each iteration to estimate transition probabilities. Y-axis is maximum absolute difference between rewards of optimal actions and actions currently implied by the algorithm. MDP sizes are $100$, probability to take random action is $1$, discount factor is $\gamma=0.95$. RB-S and Q-learning perform similarly on random MDP, but RB-S performs better on Grid World and Ring MDPs.}
\label{fig:stoch_exp1}
\end{figure*}

\begin{figure*}[t!] 
\begin{center}
\includegraphics[width=\textwidth]{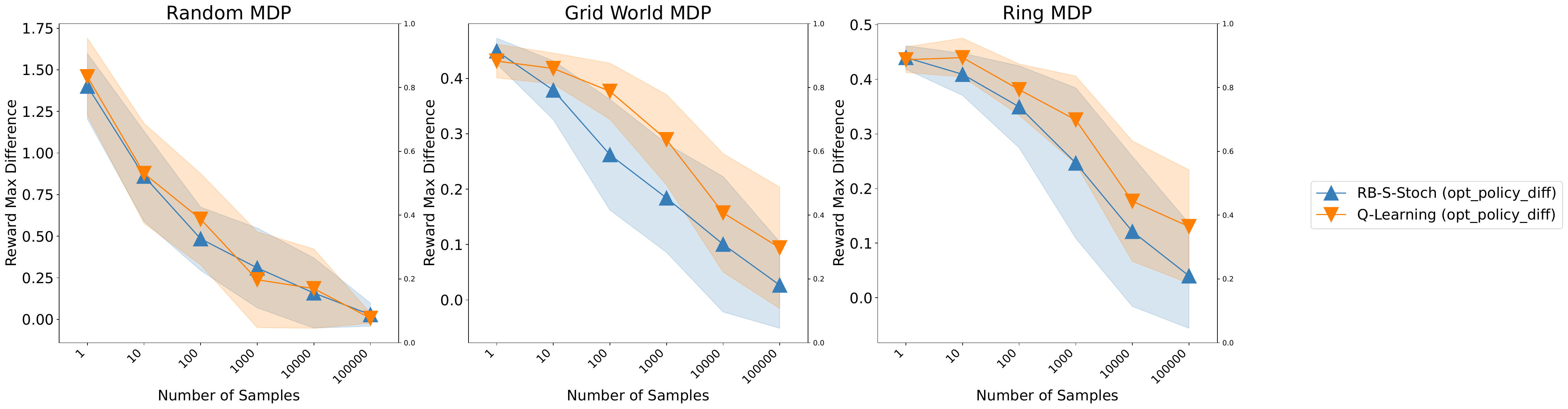}
\end{center}
\caption{Comparison of the performance of Q-learning and stochastic RB-S algorithms. X-axis is a number of samples used each iteration to estimate transition probabilities. Y-axis is maximum absolute difference between rewards of optimal actions and actions currently implied by the algorithm. MDP sizes are $200$, probability to take random action is $0.75$ and probability of self-loop action is $0.25$, discount factor is $\gamma=0.9$. RB-S and Q-learning perform similarly on random MDP, but RB-S performs better on Grid World and Ring MDPs.}
\label{fig:stoch_exp2}
\end{figure*}

\begin{figure*}[t!] 
\begin{center}
\includegraphics[width=\textwidth]{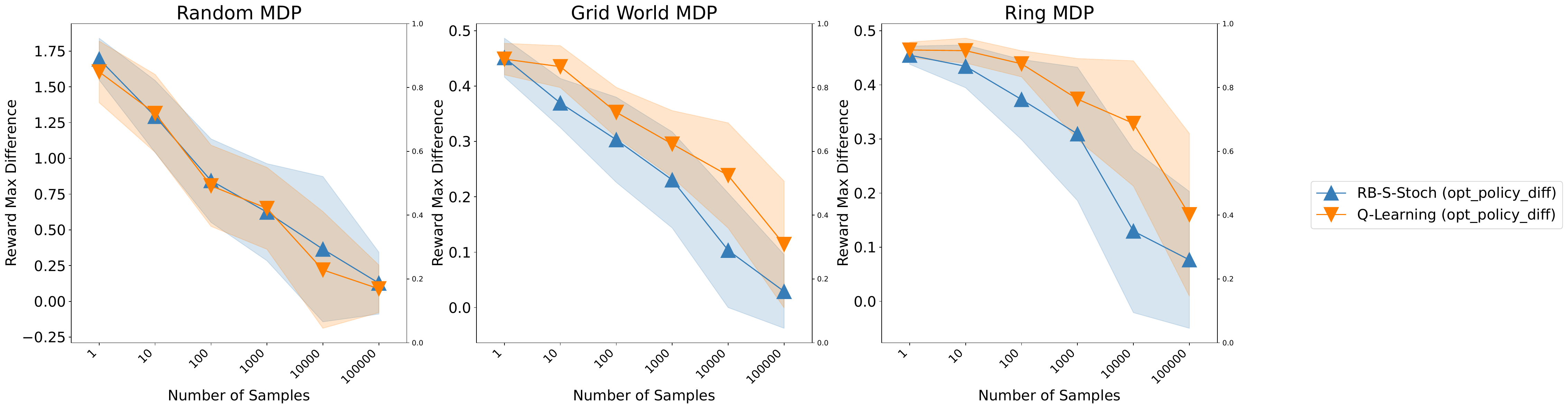}
\end{center}
\caption{Comparison of the performance of Q-learning and stochastic RB-S algorithms. X-axis is a number of samples used each iteration to estimate transition probabilities. Y-axis is maximum absolute difference between rewards of optimal actions and actions currently implied by the algorithm. MDP sizes are $400$, probability to take random action is $0.25$ and probability of self-loop action is $0.25$ and probability to execute the action is $0.5$, discount factor is $\gamma=0.9$. RB-S and Q-learning perform similarly on random MDP, but RB-S performs better on Grid World and Ring MDPs.}
\label{fig:stoch_exp3}
\end{figure*}

\end{document}